\pdfoutput=1
\documentclass[11pt]{article}
\usepackage[preprint]{acl}
\usepackage{times}
\usepackage{latexsym}
\usepackage[T1]{fontenc}
\usepackage[utf8]{inputenc}
\usepackage{microtype}
\usepackage{inconsolata}
\usepackage{graphicx}
\usepackage{amsmath,amssymb,amsthm}
\usepackage{booktabs}
\usepackage{multirow}
\usepackage{xcolor}
\usepackage{colortbl}
\usepackage{placeins}
\usepackage{adjustbox}
\usepackage{listings}
\theoremstyle{definition}
\newtheorem{proposition}{Proposition}

\newcommand{\codeurl}{\url{https://github.com/hyojuun/lfbb}}
\newcommand{\ind}[1]{\mathbf{1}\!\left[#1\right]}
\newcommand{\G}{\mathcal{G}}
\newcommand{\V}{\mathcal{V}}

\newcommand{\Method}{Ours}
\newcommand{\M}{\texttt{[M]}}
\newcommand{\KL}{\mathrm{KL}}

\title{Watermarking Diffusion Language Models with Learned Foresight and Balanced Bands}

\author{\mdseries Hyojun Kim, Hyeseon An, Yo-Sub Han$^{\dagger}$ \\
  Yonsei University, Seoul, Republic of Korea \\
  \texttt{\{\href{mailto:hyojun@yonsei.ac.kr}{hyojun}, \href{mailto:hsan@yonsei.ac.kr}{hsan}, \href{mailto:emmous@yonsei.ac.kr}{emmous}\}@yonsei.ac.kr}
}

\newcommand{\correspondingfootnote}{
    \let\oldthefootnote=\thefootnote
    \renewcommand{\thefootnote}{}
    \footnotetext{$\dagger$ Corresponding author.}
    \let\thefootnote=\oldthefootnote
}

\begin{document}
\maketitle
\correspondingfootnote

\begin{abstract}
Green-red watermarks favor tokens from a green list that a secret key derives from the preceding tokens. Diffusion language models, however, often choose a token before its preceding tokens, and the token later shapes its neighbors' green lists. Prior work therefore tilts each choice to raise the expected number of green tokens at its own and its neighbors' positions. Building on this objective, we first learn two components that prior work keeps fixed: vocabulary bands that serve as green lists, separate competing candidates and keep the null distribution exact, and a foresight predictor that estimates how each candidate shifts its neighbors' predictions. Second, we count green tokens as the detector does, once per distinct context-token pair, so repetition earns nothing. On WaterBench with LLaDA-8B-Instruct, each learned component raises the z-score at equal distortion, and their gains are nearly additive. With a strong watermark, ours detects at least 96\% of responses at a 1\% false-positive rate while matching the LLM-judge score of unwatermarked text. Ours also detects more responses on Dream and as many or more on code and question answering without a significant loss in task accuracy. Code: \codeurl{}.
\end{abstract}

\section{Introduction}
\label{sec:intro}
\begin{figure}[!t]
\centering
\includegraphics[width=\columnwidth]{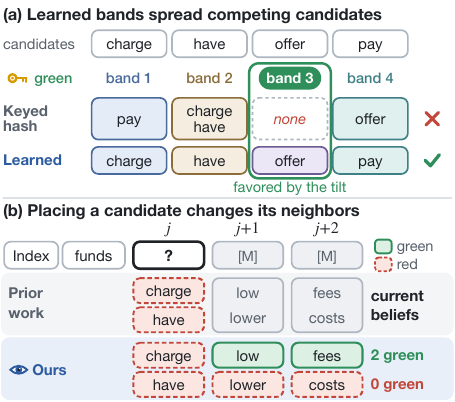}
\caption{The two learned components at one masked position (illustrative tokens). (a) Learned bands spread the competing candidates over the bands, which a keyed hash may not. (b) Predicted post-placement beliefs let foresight favor \emph{charge}, which is red but makes both neighbors green.}
\label{fig:teaser}
\end{figure}

Watermarks let the provider of a language model recognize its outputs later \citep{kirchenbauer2023watermark,dathathri2024synthid}. In green-red watermarks \citep{kirchenbauer2023watermark}, a secret key and the preceding tokens select a \emph{green list}, a subset of the vocabulary, at each position. The generator favors green tokens, and a detector that knows the key tests whether a text contains more green tokens than chance would give. This design relies on two conditions that autoregressive generation guarantees: when a token is sampled, the tokens that select its green list are already fixed, and the tokens whose green lists it helps select do not exist yet.

Diffusion language models (DLMs) such as LLaDA \citep{nie2025llada} and Dream \citep{ye2025dream}, a growing alternative to autoregressive models, break both conditions. They generate a response by iteratively unmasking a canvas of mask tokens \citep{austin2021d3pm,sahoo2024mdlm}, and they \emph{commit} tokens in an order that is not strictly left to right. In the default order of LLaDA-8B-Instruct, 39\% of the tokens are committed while part of their context is still masked, and 37\% after one of the next two tokens, whose context they are part of (Appendix~\ref{app:decoding}). A watermark must therefore choose a token before its green list is known, and this choice also shapes the green lists of its neighbors.

\citet{gloaguen2026watermarking} handle both cases with the model's predictions for the masked positions, its \emph{beliefs}. Their watermark, which we call DLM-KGW, maximizes the expected share of green tokens under a KL constraint, so it favors a candidate that is likely to be green and one that makes its neighbors green. We build on this optimization view but maximize the expected evidence of the statistic that the detector actually computes, which counts each distinct (context, token) pair, or \emph{unit}, once. The optimal sampler then \emph{tilts} the model's distribution by a score that credits a candidate for its own greenness and for its \emph{foresight}, the greenness it induces at the positions that read it as context. Because a unit that already occurs in the text adds no evidence, this \emph{unit-aware} score gives it no credit, which removes a reward for repetition that the tilt of DLM-KGW contains.

Under this objective, the evidence that a commit can gain at a given distortion depends on two components that existing green-red DLM watermarks keep fixed (Figure~\ref{fig:teaser}). The first is the \emph{partition} of the vocabulary into bands, one of which the key makes green in each context. The tilt gains the most evidence when the likely candidates are spread over the bands (Proposition~\ref{prop:balanced}), but a keyed hash puts the two most likely candidates into the same band in a quarter of the commits, and whenever another band is green, the tilt then has little to favor. The second is the beliefs from which foresight is computed. Existing methods use the neighbors' beliefs before placement, which are the same for every candidate, although placing a candidate changes the most likely token of a masked neighbor in about three quarters of uncertain commits, or \emph{decision states}.

Besides this objective, we learn both components:
\begin{itemize}
\item \emph{Learned bands} (\S\ref{sec:list}): four vocabulary bands learned from decoding traces so that competing candidates fall into different bands, while a keyed hash of the context still picks the green band, so the null distribution stays exact.
\item A \emph{foresight predictor} (\S\ref{sec:foresight}): a network initialized from the generator's top layers and trained by distillation to predict the neighbors' beliefs after placing each candidate.
\end{itemize}

We evaluate both components on the 600 WaterBench prompts \citep{tu2024waterbench} with LLaDA-8B-Instruct, against the official implementations of prior DLM watermarks (\S\ref{sec:results}). Each component raises the mean $z$-score, also under four further keys, and their gains add up. At $\delta=2$, where detection is not saturated, learned bands and the predictor raise TPR@1\% by 15 and 6 points, and by 21 together. Against DLM-KGW, ours detects up to 13 points more responses at low strengths, averaged over five keys, without a significant change in log-perplexity. At high strengths, where both detect almost every response, it keeps the judge score of unwatermarked text, whereas DLM-KGW is rated 0.5--1.0 points lower, mostly because its tilt rewards repetition loops, which the unit-aware score does not. On a second DLM, Dream, ours also detects more responses.

\begin{figure*}[t]
\centering
\includegraphics[width=\textwidth]{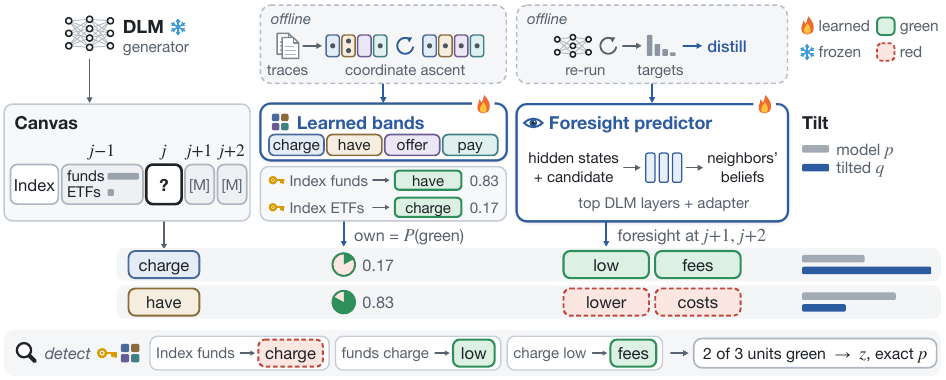}
\caption{Overview of our watermark at a masked position $j$ (illustrative tokens and probabilities). Each module is learned offline from its own data (top), the tilt combines the own and foresight terms during generation (middle), and the detector needs only the key and the bands (bottom).}
\label{fig:overview}
\end{figure*}

\section{Preliminaries}
\label{sec:background}
\paragraph{Masked diffusion language models.}
A masked DLM generates a response $x=(x_1,\dots,x_L)$ from a canvas of mask tokens \M{}. At each step, given the prompt and the committed tokens, it predicts a distribution over the vocabulary $\V$ for every masked position $i$, its \emph{belief} $\pi_i$, and \emph{commits} some positions by sampling their tokens, which then stay fixed. The decoder chooses which positions to commit, for instance the most confident ones \citep{nie2025llada}, so a position is often committed before some of its left neighbors and after some of its right ones.

\paragraph{Green-red watermarks.}
A green-red watermark \citep{kirchenbauer2023watermark} splits the vocabulary at each position $i$ into a \emph{green list} $\G(c_i)$ of relative size $\gamma$ and a red rest, seeded by a secret key $K$ and the context $c_i$, the preceding tokens. Generation adds a strength $\delta$ to the logits of green tokens, which moves probability toward them where the model is uncertain. The detector, which needs the key but not the model, counts the green tokens, and with $k$ green tokens among $n$, $z=(k-\gamma n)/\sqrt{\gamma(1-\gamma)n}$ is approximately standard normal without the watermark. Because a repeated (context, token) pair, or \emph{unit}, carries no new evidence, detectors count each distinct unit once \citep{fernandez2023three}.

\paragraph{Green-red watermarks for DLMs.}
DLM-KGW \citep{gloaguen2026watermarking} credits a candidate with its expected greenness under the beliefs about its masked context, plus the greenness it induces at the next position under that position's current belief. Maximizing the expected share of green tokens under a KL constraint then tilts the model's distribution exponentially toward candidates with high credit.

\section{Method}
\label{sec:method}

\subsection{Commits as evidence maximization}
\label{sec:formulation}
Figure~\ref{fig:overview} summarizes the method at one position. There the model prefers \emph{have}, but the tilt moves its distribution toward \emph{charge}, which is probably red itself but makes its predicted neighbors \emph{low fees} green. We take the KL-regularized view of DLM-KGW, but for the statistic that the detector computes, the number of distinct green units. We use the two preceding tokens as context, $c_j=(x_{j-2},x_{j-1})$ (Appendix~\ref{app:ablations}). Committing a candidate $v$ at position $j$ then affects the green count in two ways: the unit at $j$ may be green, and $v$ enters the contexts of the units at $j{+}1$ and $j{+}2$, its \emph{dependents}. Farther units depend on $v$ only through the beliefs, and scoring them adds little (Appendix~\ref{app:ablations}). We define the expected contribution of $v$, relative to the null rate $\gamma$, as
\begin{equation}
\label{eq:score}
s(v) = \underbrace{g_0(v) - \gamma}_{\text{own}} + \lambda\,\underbrace{\textstyle\sum_{t=1,2}\big(g_t(v)-\gamma\big)}_{\text{foresight }\Phi(v)},
\end{equation}
where $g_0(v)$ is the probability that $v$ is green given the beliefs about its context, and $g_t(v)$ is the probability that the dependent $j{+}t$ becomes green once $v$ is placed. The weight $\lambda\in[0,1]$ discounts the dependents, which are tilted again when committed. A unit that already occurs in the committed text adds no evidence, so the \emph{unit-aware} score sets its green probability to $\gamma$, at $j$ and at the dependents alike. The \emph{plain} score omits this correction and keeps rewarding repeated green units (\S\ref{sec:main}).

\begin{proposition}[KL-optimal tilt]
\label{prop:tilt}
Let $p$ be the model's distribution over the candidates at $j$, $s$ any score, and $\beta>0$. The distribution $q$ that maximizes $\mathbb{E}_{v\sim q}[s(v)]-\frac{1}{\beta}\KL(q\,\|\,p)$ is \mbox{$q^\star(v)\propto p(v)\,e^{\beta s(v)}$}.
\end{proposition}
Proposition~\ref{prop:tilt} is the standard Gibbs variational principle (proofs in Appendix~\ref{app:proofs}). The tilt is thus optimal for $s$, an estimate of one commit's evidence from the model's beliefs, not for the final detection statistic. As in \citet[Theorem~3.1]{gloaguen2026watermarking}, the tilt is applied to the logits, and we report strengths $\delta$ in their convention (Appendix~\ref{app:decoding}). Under the unit-aware score, a repeated unit adds nothing to $s$ and earns no reward, whereas a new unit multiplies the candidate's weight by a factor that exceeds one on average over keys, so the tilt favors new units, as skipping the watermark at repeated contexts does \citep{hu2024unbiased,dathathri2024synthid}. With $s(v)=\ind{v\in\G(c_j)}$, the tilt is that of \citet{kirchenbauer2023watermark}, and that of DLM-KGW is the special case with a one-token context, one dependent, $\lambda=1$, current beliefs and the plain score. The evidence gained at a given strength then depends on the partition (\S\ref{sec:list}) and on the beliefs behind foresight (\S\ref{sec:foresight}).

\subsection{Learned bands}
\label{sec:list}
Learned bands decouple the partition, which we learn, from the key, which only selects the green band. A green list learned without the key would be the same for every key, so unwatermarked text could match it, and the null distribution would no longer be known (Appendix~\ref{app:keys}). A partition assigns every token $v$ to one of $B$ bands, $\mathrm{band}(v)$. A keyed hash of the context $c$ selects one band, $b^\star(c)=H_K(c) \bmod B$, and $\G(c)=\{v:\mathrm{band}(v)=b^\star(c)\}$. As DLM-KGW does, we use $B=4$ ($\gamma=0.25$). The partition and the key are secret, and the detector needs both but not the model. While the context is still masked, $g_0(v)$ is the expectation of $\ind{\mathrm{band}(v)=b^\star(c)}$ over the beliefs about the masked context positions (Figure~\ref{fig:overview}), whereas detection depends only on the final text.

\begin{proposition}[Exact null]
\label{prop:null}
Model $H_K$ as a random function, and let $x$ be any text that does not depend on $K$, with distinct contexts $c_1,\dots,c_m$, where $c_i$ is followed by $n_{i,b}$ distinct tokens of band $b$. For any partition, the number of green units in $x$ is a sum of $m$ independent variables, the $i$-th equal to $n_{i,b}$ for a band $b$ drawn uniformly at random.
\end{proposition}
The partition still determines the gain of the tilt. If the model puts mass $M_b$ on band $b$ at a commit and $b$ is green, raising the green logits by $\beta$ raises this mass to $\rho_\beta(M_b)=M_be^{\beta}/(1-M_b+M_be^{\beta})$.

\begin{proposition}[Balanced bands]
\label{prop:balanced}
For any $\beta>0$, $\frac{1}{B}\sum_b [\rho_\beta(M_b)-M_b] \le \rho_\beta(\gamma)-\gamma$, with equality if and only if $M_b=\gamma$ for all $b$.
\end{proposition}
The gain $\rho_\beta(M)-M$ is concave in $M$ and vanishes at $M=0$ and $M=1$. A band holding almost all of the probability, or almost none, gains little when it is green, so the tilt gains the most on average over keys when every band holds a likely candidate. We therefore learn the partition that maximizes the left-hand side over the commits of unwatermarked decoding traces of training prompts, by coordinate ascent over tokens (Figure~\ref{fig:overview}, top; Appendix~\ref{app:partition}). On held-out commits, the two most likely candidates share a band in 17\% of the commits, against $1/B=25\%$ for random bands.

\subsection{Foresight predictor}
\label{sec:foresight}
Foresight needs the beliefs about the dependents once $v$ is placed, $\pi(\cdot\mid v)$. The dependent $j{+}t$ is green if its token $w$ lies in the green list of its context, which contains $v$ and the token $u$ at one other position $o$: the context of $j{+}1$ is $(x_{j-1},v)$, so $o=j{-}1$, and that of $j{+}2$ is $(v,x_{j+1})$, so $o=j{+}1$:
\begin{equation}
\label{eq:gt}
\begin{split}
g_t(v) = \textstyle\sum_{u,w}\, & \pi_{o}(u\mid v)\,\pi_{j+t}(w\mid v)\\
& \times\ind{w\in\G(c_{j+t})},
\end{split}
\end{equation}
treating the two positions as independent, as the generator's beliefs do. The current beliefs $\pi(\cdot)$ are a cheap proxy, but they are the same for every candidate, so only the green lists change with $v$. Yet placing $v$ changes the beliefs themselves, and this change accounts for most of the variation of foresight across candidates (\S\ref{sec:analysis}). Computing $\pi(\cdot\mid v)$ exactly costs one generator pass per candidate.

The \emph{foresight predictor} estimates these beliefs instead. Like the draft heads of speculative decoding \citep{cai2024medusa,li2024eagle}, it reads the generator's hidden states. It takes the states that enter the generator's top layers in a window around $j$, an adapter rewrites the state at $j$ to encode $v$, and fine-tuned copies of the top layers process the window with and without the rewrite. Their difference, added to the generator's final states and read through its output layer, gives the post-placement beliefs, which equal the current beliefs before training, and all candidates of a commit share one batched pass. We distill the predictor from the generator's own post-placement predictions at uncertain commits (decision states) of training prompts (Figure~\ref{fig:overview}, top) and apply it at uncertain commits during decoding (Appendix~\ref{app:learned}).

\subsection{Decoding and detection}
\label{sec:decoding}
We tilt the most likely candidates of every masked position of the current block, with $\lambda=\frac12$ so that the two dependents together weigh as much as the position itself. As in prior work \citep{gloaguen2026watermarking}, every masked position draws a candidate with Gumbel noise, and the low-confidence order commits the position whose draw has the highest tilted probability. We run the predictor only for the position that wins, and its foresight can change that position's draw after the choice. We therefore recompute the winner with the same noise until the winning position already carries its predicted foresight, so that the committed token is almost always the draw that won the choice (\emph{consistent order}, Appendix~\ref{app:decoding}). Special tokens, which the detector ignores, are treated as red. The detector re-tokenizes the text, counts each distinct unit once, and computes the green count's $z$-score and exact $p$-value (Proposition~\ref{prop:null}).

\section{Experimental setup}
\label{sec:setup}
\paragraph{Data and models.} We watermark LLaDA-8B-Instruct \citep{nie2025llada} on the 600 prompts of three WaterBench tasks \citep{tu2024waterbench} (WaterBench-600): long-form QA from ELI5 \citep{fan2019eli5}, finance QA from FiQA \citep{maia2018fiqa} and instruction following from AlpacaFarm \citep{dubois2023alpacafarm}. We use the prompt templates of DLM-KGW and learn the bands and the predictor on disjoint prompts from the same sources (Appendix~\ref{app:protocol}). All methods generate up to 256 tokens in blocks of 32 at temperature 0.5. They use LLaDA's default low-confidence order, which gives a far lower log-perplexity than the random order of DLM-KGW's main experiments (1.30 against 1.95 without a watermark, Appendix~\ref{app:more}).

\paragraph{Baselines.} We compare with the two DLM watermarks whose official code runs in our setting, and we generate and detect their texts with that code. DLM-KGW \citep{gloaguen2026watermarking} runs in its main configuration ($\gamma=0.25$, $\delta\in\{2,3,4,5\}$, a hash of the previous token), for which its authors found no significant difference between context sets. dgMARK \citep{hong2026dgmark} watermarks the decoding order instead of the tokens, and we run it with greedy or top-3 sampling (beam size one).

\paragraph{Metrics.} We report TPR@1\% and TPR@0.1\%, the true-positive rates at false-positive rates (FPR) of 1\% and 0.1\%, with thresholds \emph{calibrated} on the 600 unwatermarked responses of the same model and order, re-detected under 100 keys and pooled for keyed detectors so that the FPR holds on average over keys. A \emph{nominal} TPR@$x$ uses the detector's own $p$-value threshold $x$ instead. Quality is measured by the mean log-perplexity (log-PPL) under Qwen3-8B \citep{yang2025qwen3} and by an LLM judge, Gemma-4-12B-it \citep{gemma2026gemma4}, with the rubric of \citet{gloaguen2026watermarking} (mean of style and consistency, 1--10). We also report the mean $z$-score $\bar z$ and the share of responses with a repetition \emph{loop} (Rep.), in which over 15\% of the word 4-grams repeat an earlier one. Each prompt is decoded once per configuration with the same seed, and a difference is significant when its paired 95\% bootstrap confidence interval (CI) excludes zero. Besides the same strength, we compare at equal mean log-perplexity and at equal KL divergence from the model's distribution (summed over a response's commits), interpolating the reference linearly between its strengths.

\section{Results}
\label{sec:results}
\subsection{Comparison with prior DLM watermarks}
\label{sec:main}
\begin{table}[t]
\centering
\small
\setlength{\tabcolsep}{2.6pt}
\begin{tabular}{@{}lcrrrrr@{}}
\toprule
 & & \multicolumn{2}{c}{TPR (\%) $\uparrow$} & \multicolumn{1}{c}{log-PPL} & \multicolumn{1}{c}{Judge} & \multicolumn{1}{c}{Rep.} \\
\cmidrule(lr){3-4}
Method & $\delta$ & @1\% & @0.1\% & $\downarrow$ & $\uparrow$ & $\downarrow$ \\
\midrule
No watermark & -- & -- & -- & 1.297 & 9.12 & 19.3 \\
dgMARK, greedy & -- & 92.7 & 84.2 & 1.550 & 9.07 & 16.2 \\
dgMARK, top-3 & -- & 92.3 & 84.3 & 1.557 & 9.12 & 15.2 \\
\midrule
DLM-KGW & 2 & 51.3 & 24.8 & \textbf{1.316} & 9.13 & 19.3 \\
\Method{} & 2 & \textbf{69.5} & \textbf{44.0} & 1.326 & \textbf{9.29} & \textbf{18.2} \\
\midrule
DLM-KGW & 3 & 85.7 & 70.0 & \textbf{1.387} & 8.97 & 23.3 \\
\Method{} & 3 & \textbf{89.3} & \textbf{79.3} & 1.398 & \textbf{9.30} & \textbf{14.0} \\
\midrule
DLM-KGW & 4 & 96.0 & 91.3 & 1.506 & 8.62 & 26.3 \\
\Method{} & 4 & \textbf{96.2} & \textbf{92.7} & \textbf{1.469} & \textbf{9.32} & \textbf{14.3} \\
\midrule
DLM-KGW & 5 & 97.7 & 95.3 & 1.602 & 8.15 & 30.5 \\
\Method{} & 5 & \textbf{98.3} & \textbf{96.0} & \textbf{1.575} & \textbf{9.26} & \textbf{12.2} \\
\bottomrule
\end{tabular}
\caption{Main results on WaterBench-600 with LLaDA-8B-Instruct in its low-confidence order. TPRs use calibrated thresholds, log-PPL is measured under Qwen3-8B, the judge is Gemma-4-12B-it (1--10), and Rep. is the share of responses with a repetition loop (\S\ref{sec:setup}). Bold marks the better of DLM-KGW and ours at the same $\delta$.}
\label{tab:main}
\end{table}

\begin{figure}[t]
\centering
\includegraphics[width=\columnwidth]{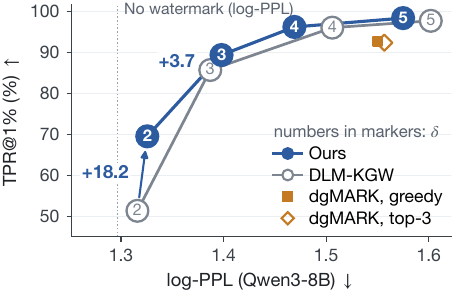}
\caption{TPR@1\% against log-PPL on WaterBench-600 with LLaDA-8B-Instruct in its low-confidence order. Arrows mark our gain at the same $\delta$.}
\label{fig:frontier}
\end{figure}
\IfFileExists{tables/table_fac.tex}{\begin{table*}[t]
\centering
\small
\setlength{\tabcolsep}{3.3pt}
\begin{tabular}{@{}llrrrrrrrrrrrr@{}}
\toprule
 & & \multicolumn{4}{c}{$\delta=2$} & \multicolumn{4}{c}{$\delta=3$} & \multicolumn{4}{c}{$\delta=4$} \\
\cmidrule(lr){3-6}\cmidrule(lr){7-10}\cmidrule(lr){11-14}
 & & $\bar z$ & TPR & log-PPL & Judge & $\bar z$ & TPR & log-PPL & Judge & $\bar z$ & TPR & log-PPL & Judge \\
Partition & Foresight & $\uparrow$ & @1\% $\uparrow$ & $\downarrow$ & $\uparrow$ & $\uparrow$ & @1\% $\uparrow$ & $\downarrow$ & $\uparrow$ & $\uparrow$ & @1\% $\uparrow$ & $\downarrow$ & $\uparrow$ \\
\midrule
 & Current beliefs & 2.35 & 49.0 & 1.337 & 9.18 & 3.78 & 84.5 & 1.397 & \textbf{9.35} & 5.27 & 94.5 & 1.507 & \textbf{9.40} \\
\multirow{-2}{*}{\begin{tabular}[c]{@{}l@{}}Keyed\\hash\end{tabular}} & Predictor & 2.61 & 59.5 & 1.351 & \textbf{9.29} & 3.98 & 86.5 & 1.403 & 9.31 & 5.45 & 95.7 & 1.489 & 9.39 \\
\midrule
 & None ($\lambda=0$) & 2.06 & 39.3 & 1.317 & 9.22 & 3.11 & 70.7 & 1.372 & 9.30 & 4.15 & 89.5 & 1.451 & 9.22 \\
 & Current beliefs & 2.77 & 63.7 & \textbf{1.308} & 9.24 & 4.25 & 88.7 & \textbf{1.391} & 9.33 & 5.69 & \textbf{96.3} & 1.482 & 9.25 \\
\rowcolor{black!7} & Predictor & \textbf{2.95} & \textbf{69.5} & 1.326 & \textbf{9.29} & \textbf{4.35} & \textbf{89.3} & 1.398 & 9.30 & \textbf{5.97} & 96.2 & \textbf{1.469} & 9.32 \\
\multirow{-4}{*}{\begin{tabular}[c]{@{}l@{}}Learned\\bands\end{tabular}} & Exact$^\dagger$ & 3.10 & 73.2 & 1.326 & -- & 4.66 & 93.3 & 1.378 & 9.27 & 6.11 & 96.8 & 1.469 & 9.31 \\
\bottomrule
\end{tabular}
\caption{Crossing the two learned components on WaterBench-600 with LLaDA-8B-Instruct (same prompts, seeds and unit-aware score in every row). The keyed hash assigns bands by a keyed hash of the context and the token, as standard green-red watermarks do, and exact foresight$^\dagger$, a reference, re-runs the generator for each candidate. Metrics as in Table~\ref{tab:main}. The shaded row is ours, and bold marks the best crossed cell at each $\delta$. Further metrics are in Table~\ref{tab:facfull} and paired 95\% CIs in Table~\ref{tab:contrasts}.}
\label{tab:fac}
\end{table*}
}{}

\paragraph{At low strengths, ours detects more at the same log-perplexity.} Low strengths cost the least quality but leave detection uncertain, so extra evidence matters most there. At $\delta=2$ and 3, ours detects more responses than DLM-KGW with no significant difference in log-perplexity (Tables~\ref{tab:main} and~\ref{tab:paired}, Figure~\ref{fig:frontier}). Averaged over five keys on 300 prompts, since one key shifts each detector's null distribution differently (Appendix~\ref{app:keys}), ours detects 12.7 and 4.5 points more responses at 1\% FPR and 11.9 and 10.3 points more at 0.1\% FPR, all significantly (Table~\ref{tab:mkey}). At equal log-perplexity, all eight differences favor ours, four significantly (Appendix~\ref{app:more}). Ours thus turns the same strength into more evidence, a gain that \S\ref{sec:components} traces to the two learned components. Both versions of dgMARK detect fewer responses than ours at a higher log-perplexity.

\paragraph{Where both detect almost every response, ours keeps the quality of text without a watermark.} At $\delta=4$ and 5, where both detect 96--98\% of the responses, the judge rates ours 9.32 and 9.26, at least as high as text without a watermark (9.12), against 8.62 and 8.15 for DLM-KGW, and a second judge agrees (Appendix~\ref{app:more}). Most of the gap is explained by repetition loops. The tilt of DLM-KGW keeps rewarding repeated green units, although both detectors count each unit once, so it spends strength on evidence that is never counted, and its share of looping responses grows from 19 to 31\% with $\delta$, while ours stays between 12 and 18\%. Removing our unit-aware score confirms this cause, as ours then loops in 21\% of its responses instead of 14\% at $\delta=3$ and 4 and is rated 0.35 and 0.44 lower (Appendix~\ref{app:contrasts}). On the prompts where neither response loops, the judge rates the two watermarks alike at $\delta\le4$ (Table~\ref{tab:loopfree}). Rewarding only the evidence that the detector counts thus keeps quality at the strengths that make detection reliable.

\subsection{The two learned components}
\label{sec:components}
Table~\ref{tab:fac} crosses the partition (a keyed hash of the context and the token, as in standard watermarks, or learned bands) with the beliefs behind foresight (current or predicted) under the same objective. With neither learned component, our objective detects about as many responses as DLM-KGW at the same log-perplexity and is already rated 9.35 and 9.40 at $\delta=3$ and 4. The quality advantage over DLM-KGW therefore comes from the objective, whereas the learned components add detection, also without the unit-aware score (Appendix~\ref{app:contrasts}).

\paragraph{Learned bands raise the mean $z$-score and keep the null exact.} The tilt can prefer a green candidate only if the competing candidates fall into different bands, and learned bands put the two most likely candidates into the same band in 17\% of held-out commits, against a quarter for a keyed hash (\S\ref{sec:list}). Over the cell with neither component, they raise the mean $z$-score by 0.41--0.47 at every strength without raising log-perplexity, and at $\delta=2$, where detection is not saturated, they detect 14.7 points more responses at 1\% FPR. The judge score drops only at $\delta=4$, by 0.15, which the second judge does not confirm (Appendix~\ref{app:contrasts}). The gain comes from learning, not from fixing each token's band, since random bands raise the mean $z$-score by only 0.08 at $\delta=3$ (not significant). As the key still picks the green band, learning costs no false positives, and on unwatermarked text learned bands flag 0.7--1.8\% of texts at a nominal 1\% level, none significantly above 1\% (Appendix~\ref{app:keys}).

\paragraph{The foresight predictor raises the mean $z$-score with either partition.} Foresight rewards a candidate for making its neighbors green, so it is only as good as its beliefs about the neighbors after the placement. The predictor estimates these beliefs instead of reusing the current ones (\S\ref{sec:analysis}) and raises the mean $z$-score by 0.10--0.28 at the same strength with either partition, significantly in five of the six cells. At equal KL, the quantity the tilt trades against (Proposition~\ref{prop:tilt}), its gain is significant at every strength (Figure~\ref{fig:foresight}b), and with learned bands it needs 14 and 8\% less KL for the same mean $z$-score at $\delta=2$ and 3. At $\delta=2$, it detects 5.8 points more responses at 1\% FPR, significantly, whereas at $\delta=3$ and 4, where detection is nearly saturated, its TPR@1\% changes by less than a point.

\paragraph{Their gains add up.} The two components improve different inputs of the score, the partition which candidates can be green and the predictor how well foresight sees the neighbors, and their gains in mean $z$-score are close to additive, with no significant interaction (Appendix~\ref{app:contrasts}). Together, they detect 20.5 and 4.8 points more responses at 1\% FPR at $\delta=2$ and 3, with no significant change in log-perplexity or judge score, and need 30 and 22\% less KL for the same mean $z$-score (Figure~\ref{fig:components}). Both gains hold over four further keys (Appendix~\ref{app:contrasts}). Detection never runs the predictor, so these gains leave it as cheap as with a keyed hash.

\subsection{Why post-placement beliefs matter}
\label{sec:analysis}
\begin{figure}[t]
\centering
\includegraphics[width=\columnwidth]{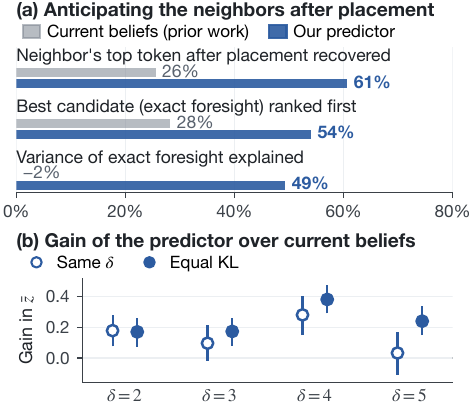}
\caption{(a) Fidelity to exact foresight on 2{,}397 held-out decision states (Table~\ref{tab:fidelity}). (b) Gain of the predictor in $\bar z$ over current-belief foresight with learned bands on WaterBench-600, at the same $\delta$ and at equal KL, with 95\% CIs.}
\label{fig:foresight}
\end{figure}
On 2{,}397 held-out decision states (Appendix~\ref{app:learned}), placing a candidate changes the most likely token of a masked neighbor in 74\% of the cases (examples in Table~\ref{tab:examples}). Current-belief foresight cannot see this change, so it is only weakly correlated with exact foresight (correlation 0.39) and explains none of its variance (Figure~\ref{fig:foresight}a). The predictor recovers each neighbor's post-placement top-1 token in 61\% of the cases, explains 49\% of the variance, and ranks the exact best candidate first in 54\% of the states, against 28\% for current beliefs.

The predictor's gain in detection is smaller than its gain in fidelity because foresight has limited leverage on the tilt. It enters the score with weight $\frac12$, the predictor is used at a third of the tilts, and the dependents it favors are tilted again when committed. Even exact foresight, which re-runs the generator for each of the top-8 candidates, reaches a mean $z$-score of only 4.66 and 6.11 at $\delta=3$ and 4 at about the same KL. The predictor closes a quarter and two thirds of this gap, at 1.49 times the decoding time of current beliefs against at least 2.15 times for exact foresight (Appendix~\ref{app:decoding}).

\subsection{Other tasks, another model, robustness}
\label{sec:dream}
\IfFileExists{tables/table_transfer.tex}{\begin{table*}[t]
\small
\begin{minipage}[t]{0.52\textwidth}
\centering
\setlength{\tabcolsep}{4.2pt}
\begin{adjustbox}{max width=\linewidth}
\begin{tabular}{@{}lcccc@{}}
\toprule
 & \multicolumn{2}{c}{In distribution} & \multicolumn{2}{c}{Out of distribution} \\
\cmidrule(lr){2-3}\cmidrule(l){4-5}
Method & MBPP & MMLU-Pro & HumanEval & GPQA-D \\
\midrule
\multicolumn{5}{@{}l}{\emph{TPR@1\% (\%)} $\uparrow$} \\
DLM-KGW & 15.0 & 60.0 & 17.7 & 72.2 \\
Ours & \textbf{22.5} & \textbf{65.0} & 17.7 & 72.2 \\
\midrule
\multicolumn{5}{@{}l}{\emph{Pass@1 or accuracy (\%)} $\uparrow$} \\
No watermark & 37.0 & 40.0 & 35.4 & 20.7 \\
DLM-KGW & 32.5 & 35.0 & \textbf{34.8} & 26.8 \\
Ours & \textbf{36.5} & \textbf{35.7} & 33.5 & \textbf{29.3} \\
\bottomrule
\end{tabular}
\end{adjustbox}
\caption{Code and QA with LLaDA-8B-Instruct at $\delta=4$, with both components learned on training problems from MBPP, MMLU-Pro and GSM8K, none from HumanEval or GPQA-Diamond (GPQA-D, chance 25\%). Bold: the better watermark.}
\label{tab:transfer_tasks}
\end{minipage}
\hfill
\begin{minipage}[t]{0.46\textwidth}
\centering
\setlength{\tabcolsep}{3.2pt}
\begin{adjustbox}{max width=\linewidth}
\begin{tabular}{@{}lcrrr@{}}
\toprule
 & & TPR@1\% (\%) & Valid (\%) & log-PPL \\
Method & $\delta$ & $\uparrow$ & $\uparrow$ & $\downarrow$ \\
\midrule
No watermark & -- & -- & 62.8 & 2.188 \\
\midrule
DLM-KGW & 3 & 42.0 & 53.8 & 2.505 \\
Ours, transferred & 3 & 57.3 & 78.2 & 2.290 \\
Ours, learned & 3 & \textbf{60.0} & \textbf{81.3} & \textbf{2.278} \\
\midrule
DLM-KGW & 4 & 56.5 & 54.3 & 2.631 \\
Ours, transferred & 4 & 75.5 & 80.3 & 2.353 \\
Ours, learned & 4 & \textbf{78.3} & \textbf{84.2} & \textbf{2.316} \\
\bottomrule
\end{tabular}
\end{adjustbox}
\caption{Dream-v0-Instruct-7B on WaterBench-600 (official random order), nominal TPR@1\%. Transferred: LLaDA's partition with current beliefs. Learned: both components learned for Dream. log-PPL over valid responses.}
\label{tab:transfer_dream}
\end{minipage}
\end{table*}
}{}
\paragraph{Other tasks.} To test whether the components can be learned for another domain, we learn both on code and QA training prompts (Appendix~\ref{app:codeqa}) and evaluate them at $\delta=4$ on held-out MBPP and MMLU-Pro problems \citep{austin2021mbpp,wang2024mmlupro} and on HumanEval and GPQA-Diamond \citep{chen2021humaneval,rein2024gpqa}, which have no training prompts. On the tasks they were learned for, ours detects 7.5 points more programs than DLM-KGW (significantly) and 5.0 points more answers, and on the unseen tasks as many (Table~\ref{tab:transfer_tasks}), while its pass@1 and accuracy do not fall significantly below those without a watermark. Both detect fewer programs than answers, as programs leave few tokens uncertain. Three other writing tasks test instead whether components learned on WaterBench-style prompts carry over unchanged. On their 900 prompts, ours detects 6.0 points more than DLM-KGW at $\delta=2$ and as many at $\delta=3$, and at $\delta=4$, where both detect over 90\%, it detects 3.7 fewer but loops in 10\% of its responses instead of 26\% (Appendix~\ref{app:ood}).

\paragraph{Dream.} On Dream-v0-Instruct-7B (official random order, plain score), we test whether a learned partition transfers across models. LLaDA's partition, transferred through identical tokens with current-belief foresight, already detects 57.3 and 75.5\% of the 600 responses at $\delta=3$ and 4 (nominal TPR@1\%), against 42.0 and 56.5\% for DLM-KGW, and both components learned for Dream detect 60.0 and 78.3\% (Table~\ref{tab:transfer_dream}). Most of the gap to DLM-KGW comes from response length rather than from the signal, since 44\% of its responses at $\delta=4$ are too short to carry a signal, against 13\% of ours, which tilts the end-of-text token as red. Learning for Dream helps through the predictor, which raises the mean $z$-score over current-belief foresight by 0.22 and 0.15 (both significant), as on LLaDA.

\paragraph{Robustness and decoding order.} Under 20\% random deletions or insertions, ours detects about as many responses as DLM-KGW in the same order, and under every token edit its AUC stays above 96\%. Under 20\% substitutions it detects 6.7 points fewer, the price of two-token contexts, which trade robustness for detection without edits (Appendix~\ref{app:ablations}). After paraphrasing, both token-level watermarks keep an AUC of 67--75\%, against 52--58\% for dgMARK (Appendix~\ref{app:robust}). In the random order, on which its predictor is not trained, ours still detects about as many responses at a lower log-perplexity (Appendix~\ref{app:more}).

\section{Related work}
\label{sec:related}
\paragraph{Watermarking diffusion LMs.} One line keeps token-level green lists and handles contexts that are not yet decoded. Besides DLM-KGW (\S\ref{sec:background}), methods use the most likely tokens of masked neighbors \citep{dmark2025}, only the already revealed neighbors \citep{raban2026lrdwm}, or watermark a position when an adjacent token is unmasked \citep{yang2026umr}. Another line uses other channels, such as the decoding order \citep{hong2026dgmark}, position-keyed sampling noise \citep{bagchi2025}, key patterns for order-agnostic models \citep{chen2025patternmark}, or per-token signal vectors that carry several bits \citep{huang2026cdmark}. Methods of the first line take the partition from a keyed hash and, when they look ahead, use the neighbors' beliefs before placement. Ours belongs to this line but learns both, a partition that separates competing candidates and the neighbors' beliefs after placement.

\paragraph{Designing green lists.} Green-red watermarks bias the distribution, unlike distortion-free schemes that couple the sampling randomness to the key \citep{christ2024undetectable,kuditipudi2024robust,dathathri2024synthid}, and several change which tokens share a list to lose less quality. WatME \citep{chen2024watme} and \citet{guo2024semantic} split near-synonyms between green and red, and Balance-Marking \citep{wang2024codable} and WaterMod \citep{park2026watermod} balance the probability of the lists with a proxy model or by probability rank. Token-specific watermarking \citep{huo2024tsw} and SIR \citep{liu2024sir} learn networks that set each token's strength or bias, so their detectors need a language model or the learned network. Unigram-Watermark \citep{zhao2024unigram} fixes one random green list for every position, and a distribution-preserving watermark chooses among random vocabulary segments \citep{chen2025mcmark}. Unlike all of these, learned bands fit the partition to the candidates that compete during decoding, while a keyed hash of each context still picks the green band, so the green list changes with the context and detection needs only the key and the partition.

\section{Conclusion}
\label{sec:conclusion}
Viewing each commit of a DLM as maximizing the evidence its detector expects shows that a green-red watermark depends on two components that prior work keeps fixed: the partition of the vocabulary and the beliefs behind foresight. Learning a partition that separates competing candidates and a predictor of the neighbors' post-placement beliefs each raises detection, with additive gains, while counting each unit once, as the detector does, keeps the quality of unwatermarked text. Because the key still chooses the green band and detection never runs the predictor, a provider gains this detection and quality while keeping what makes green-red detection practical, an exact null distribution and a detector that needs only the key and the bands.

\section*{Limitations}
\textbf{Security.} The bands are learned, but the key chooses which band is green in each context, so the null distribution stays exact for any partition (Proposition~\ref{prop:null}). Since every green list is one of four fixed sets, the partition should stay secret together with the key, and attacks that target the partition are left for future work. \textbf{Robustness.} Like other green-red watermarks, ours loses most of its signal under paraphrasing. Under random deletions and insertions it detects about as many responses as DLM-KGW, and the width of its contexts trades detection without edits against robustness to token substitutions (Appendices~\ref{app:ablations} and~\ref{app:robust}). \textbf{Time and cost.} The predictor is trained once for each generator, which takes about 29 GPU-hours including its training data (Appendix~\ref{app:resources}), and it makes generation about 1.5 times slower than foresight from the current beliefs. Exact foresight, which re-runs the generator for each candidate, is at least 2.15 times slower. The predictor lies between the two and closes a quarter to two thirds of the gap between them in mean $z$-score (\S\ref{sec:analysis}). \textbf{Evaluation.} Our judges are open models rather than human raters. To limit the bias of a single judge, two judges of different families rate all methods with the same instructions, and their scores correlate at 0.84 over responses and 0.99 over methods (Appendix~\ref{app:more}).

\section*{Ethical considerations}
Watermarking supports the attribution of machine-generated text and can deter misuse, but it can also be used to track users, and false positives can harm people whose writing is flagged. We report false-positive rates on human-written text and do not recommend acting on a single detection without other evidence.

\bibliography{references}

\appendix
\renewcommand{\topfraction}{0.9}\renewcommand{\dbltopfraction}{0.9}\renewcommand{\textfraction}{0.07}\renewcommand{\floatpagefraction}{0.75}\renewcommand{\dblfloatpagefraction}{0.75}\setcounter{topnumber}{4}\setcounter{dbltopnumber}{3}\setcounter{totalnumber}{6}
\noindent Appendix~\ref{app:protocol} gives the experimental protocol, and the other appendices follow the order of the paper. Appendix~\ref{app:proofs} proves the propositions. Appendices~\ref{app:impl} and~\ref{app:learned} detail the partition, the decoding and the foresight predictor, with examples of post-placement beliefs. Appendix~\ref{app:keys} covers the null distribution and the calibration over keys, and Appendices~\ref{app:more}--\ref{app:ablations} give further results and ablations on WaterBench. Appendices~\ref{app:codeqa}--\ref{app:robust} cover code and question answering, Dream, other writing tasks and robustness, Appendix~\ref{app:resources} lists the resources we used, and Appendix~\ref{app:prompts} gives the judges' prompts.

\section{Experimental protocol}
\label{app:protocol}
\paragraph{Prompts.} Table~\ref{tab:splits} lists every prompt set. All are in English. The partition and the predictor are learned on training prompts from the same three sources as WaterBench, disjoint from the 600 evaluation prompts and filtered for near-duplicates of them. Appendix~\ref{app:prompts} gives the judges' instructions.

\begin{table}[h]
\centering
\small
\setlength{\tabcolsep}{3pt}
\begin{tabular}{@{}llr@{}}
\toprule
Purpose & Source & Prompts \\
\midrule
\multicolumn{3}{@{}l}{\emph{Training}} \\
Partition & ELI5, FiQA, AlpacaFarm & 1{,}200 \\
Predictor & ELI5, FiQA, AlpacaFarm & 1{,}612 \\
Code/QA (both) & MBPP, MMLU-Pro, GSM8K & 1{,}791 \\
\midrule
\multicolumn{3}{@{}l}{\emph{Development}} \\
Context width & ELI5, FiQA, AlpacaFarm & 150 \\
Foresight weight & ELI5, FiQA, AlpacaFarm & 150 \\
\midrule
\multicolumn{3}{@{}l}{\emph{Evaluation}} \\
Main results & WaterBench & 600 \\
Robustness & WaterBench (subset) & 90 \\
Code & HumanEval, MBPP & 164, 200 \\
QA & GPQA-D, MMLU-Pro & 198, 140 \\
\bottomrule
\end{tabular}
\caption{Prompt sets. All training prompts are disjoint from the evaluation prompts. Those of the partition and the predictor come from the training splits of the WaterBench sources, and the development prompts from their development splits. Each predictor prompt is decoded once, at one of four strengths ($\delta'\in\{1,1.5,2,2.5\}$, about 400 prompts each). Code/QA (both): the partition and the predictor learned for code and QA (Table~\ref{tab:transfer_tasks}, Appendix~\ref{app:codeqa}). QA: question answering. GPQA-D: GPQA-Diamond. Robustness: three seeds per prompt.}
\label{tab:splits}
\end{table}

\paragraph{Decoding.} With LLaDA, every method generates up to 256 tokens in blocks of 32, one token per step, at temperature 0.5 (Dream: Appendix~\ref{app:codeqa}). Candidates are drawn with Gumbel noise from a stream determined by the seed and the step, so all configurations use the same noise for the same prompt. Batched bf16 computation on GPUs still makes repeated runs of one configuration differ slightly.

\paragraph{Metrics.} Thresholds at 1\% and 0.1\% FPR are the empirical quantiles of the scores of the 600 unwatermarked responses of the same model and order. For keyed detectors the responses are re-detected with 100 keys and pooled (Appendix~\ref{app:keys}). Log-perplexity is computed under Qwen3-8B given the prompt. The judge, Gemma-4-12B-it, receives the system prompt of \citet{gloaguen2026watermarking} and the user message ``[Question] \emph{prompt} [Answer] \emph{response} [End Answer]'', and scores style, consistency, accuracy and ethics from 1 to 10. We report the mean of style and consistency, as they do. The judge decodes greedily, and a reply without the style, consistency and accuracy scores is regenerated up to twice at temperature 0.7. A response has a repetition \emph{loop} when the share of its word 4-grams that repeat an earlier one exceeds 0.15. As sanity checks, the judge rates responses with a loop about 7 and the others about 9.6, and its score for DLM-KGW falls monotonically with $\delta$ (Table~\ref{tab:main}). KL is summed over a response's commits and does not count how the tilt changes which position is committed next. Log-perplexity, measured on the final text, includes this effect.

\paragraph{Statistics.} Differences between configurations are paired by prompt (and seed). The 95\% CIs are percentile intervals of 2{,}000--4{,}000 bootstrap resamples of the prompts. Comparisons at equal log-perplexity or equal KL interpolate the reference configuration linearly between its strengths within each bootstrap resample. Resamples in which the compared value falls outside the reference's range are dropped.

\FloatBarrier
\section{Proofs}
\label{app:proofs}
\begin{proof}[Proof of Proposition~\ref{prop:tilt}]
Let $Z=\sum_v p(v)\,e^{\beta s(v)}$, so that $q^\star(v)=p(v)\,e^{\beta s(v)}/Z$. For any distribution $q$ over the candidates,
\[
\begin{aligned}\KL(q\,\|\,q^\star)&=\sum_v q(v)\log\frac{q(v)\,Z}{p(v)\,e^{\beta s(v)}}\\&=\KL(q\,\|\,p)-\beta\,\mathbb{E}_{v\sim q}[s(v)]+\log Z,\end{aligned}
\]
so the objective equals $\frac{1}{\beta}\big(\log Z-\KL(q\,\|\,q^\star)\big)$. Since $\KL(q\,\|\,q^\star)\ge 0$ with equality only at $q=q^\star$, the objective is maximized by $q^\star$ and only by it (the Gibbs variational principle).
\end{proof}
Proposition~\ref{prop:tilt} concerns one commit, with the base distribution and the score fixed. Our sampler applies it at every commit, so it optimizes the objective commit by commit rather than over the whole decoding, and the KL we report sums the divergences of these commits.

\begin{proof}[Proof of Proposition~\ref{prop:null}]
A unit $(c_i,v)$ is green if and only if $\mathrm{band}(v)=b^\star(c_i)$, so the number of green units that follow $c_i$ is $n_{i,b^\star(c_i)}$, and the number of green units in $x$ is $\sum_i n_{i,b^\star(c_i)}$. The counts $n_{i,b}$ depend only on $x$ and the partition, which do not depend on $K$. Because $H_K$ is a random function and the contexts $c_i$ are distinct, the bands $b^\star(c_i)=H_K(c_i)\bmod B$ are independent and uniform over the $B$ bands. Hence the terms $n_{i,b^\star(c_i)}$ are independent, and each equals $n_{i,b}$ for a uniformly random band $b$.
\end{proof}

\begin{proof}[Proof of Proposition~\ref{prop:balanced}]
With $a=e^{\beta}>1$, $\rho_\beta(M)=aM/(1+(a-1)M)$ has second derivative $-2a(a-1)/(1+(a-1)M)^3<0$, so it is strictly concave on $[0,1]$. The masses $M_b$ sum to one, so they average to $1/B=\gamma$. By Jensen's inequality, $\frac{1}{B}\sum_b\rho_\beta(M_b)\le\rho_\beta(\gamma)$, with equality if and only if all $M_b$ are equal, that is, $M_b=\gamma$ for all $b$. Subtracting $\frac{1}{B}\sum_b M_b=\gamma$ from both sides gives the claim.
\end{proof}

\FloatBarrier
\section{Method details}
\label{app:impl}
\subsection{Learning the partition}
\label{app:partition}
\paragraph{Data.} We decode the 1{,}200 training prompts (400 per task) once with LLaDA-8B-Instruct without a watermark in the random order, and 210 of them also in the low-confidence order, with the decoding settings of Appendix~\ref{app:protocol}. Each commit before the end-of-text token gives a record with the top-50 candidates and their probabilities at temperature 0.5. We hold out 5\% of the random-order prompts (about 10{,}000 of 253{,}000 records). Most traces come from the random order, but the candidates of a commit differ little between the two orders, and the learned bands raise detection in the low-confidence order (Table~\ref{tab:fac}).

\paragraph{Objective and algorithm.} The objective is $\frac1B\sum_b \rho_\beta(M_b)$ summed over the training records whose most likely candidate has probability below 0.98, with $M_b$ the share of the probability of the record's top-50 candidates, renormalized over these 50, that falls in band $b$ (so the $M_b$ sum to one, as in Proposition~\ref{prop:balanced}), and $\beta=4$. The search starts from a random partition, which assigns each token to a band by a fixed hash of its id. Each pass visits every candidate token once, in decreasing total probability, and moves it to the band that most increases the objective over the records in which it is a candidate. Passes repeat until no token moves (a few minutes on one CPU core). Tokens that never appear as candidates keep their initial band. The partition is a table with one band per token. The key enters only through $b^\star(c)$, a keyed integer hash of the two context token ids reduced modulo $B$. The hash applies rounds of multiplication and xor-shift, so that every bit of the result depends on the key and on both ids.

\paragraph{Held-out gain and alternatives.} On held-out records, the average gain $\frac1B\sum_b[\rho_4(M_b)-M_b]$ is 0.432 for the learned partition, 0.413 for a band network (an MLP over the context and token embeddings, trained toward balanced partitions) and 0.402 for a keyed hash, with the same ranking at $\beta=2$. The network separates competing candidates less, because competing candidates have similar embeddings. Under the network, the two most likely candidates share a band in 21\% of the commits, against 17\% under the learned partition. A soft relaxation of the network's band assignment and per-context refinements of the learned partition did not raise the held-out gain.

\paragraph{Balance across bands.} Learned bands do not gather the likely tokens in some bands. On the same held-out records, each band holds 24--26\% of the model's probability mass and of the most likely candidates, and on the commits with $p_{\max}<0.98$ the expected gain when a band is green is 0.43--0.44 for every band. One band holds more than 90\% of a commit's mass in 27\% of these commits, against 30\% for a keyed hash. A fixed hash of the token id, by contrast, puts 36\% of the mass into one band, because a few frequent tokens carry much of it.

\paragraph{Secrecy of the partition.} Many partitions are about equally good, and which one the search reaches depends mostly on where it starts. Starting from two other random partitions, the same procedure reaches nearly the same held-out gain (0.429 and 0.430, against 0.432), but the resulting bands agree with ours little more than a random partition does. After the best relabeling of the bands, 34 and 39\% of the candidates' probability mass on held-out records falls into the same band as ours, against 34\% for a random partition (weighted adjusted Rand index 0.03 and 0.05). A provider can therefore keep its watermark private by keeping the key, the initialization of the search and the learned partition secret: an attacker who re-runs our procedure on the open generator without them obtains a different partition.

\subsection{Decoding}
\label{app:decoding}
\paragraph{Strength.} As \citet{gloaguen2026watermarking} do, we add $\delta' s(v)$ to the logits before the temperature $T$ is applied, so $\beta=\delta'/T$ in Proposition~\ref{prop:tilt}. Their implementation scales its two terms by $\delta/2$, so we set $\delta'=\delta/2$ and report $\delta$ in their convention.

\paragraph{Commit order.} In unwatermarked decoding of the training prompts, LLaDA-8B-Instruct in its low-confidence order commits 39.2\% of its tokens while position $j{-}1$ or $j{-}2$ is still masked, and 36.7\% after position $j{+}1$ or $j{+}2$ has been committed (12.5\% both). Dream-v0-Instruct-7B in its random order does so for 65.9\% and 66.3\% of its tokens (35.9\% both). An autoregressive model does neither.

\paragraph{Foresight from the current beliefs.} Without placing $v$, the probability that the dependent $j{+}t$ becomes green is approximated from the current beliefs,
\begin{equation}
\label{eq:naive}
g_t(v)\approx\textstyle\sum_{u} \pi_{o}(u)\sum_{w} \pi_{j+t}(w)\,\ind{w\in\G(c_{j+t})},
\end{equation}
with $c_{j+t}$, $u$ and $w$ as in Eq.~\eqref{eq:gt}. Only the green list depends on $v$. With the predictor, $\pi_{j+t}$ and $\pi_o$ are replaced by the predicted beliefs with $v$ in place.

\paragraph{Unit-aware score.} A unit $(c,w)$ counts as present when a committed position of the response holds the same context and token (masked positions are never compared, and the first two positions use the padding context, as the detector does). Its green probability is then replaced by $\gamma$. This applies to the unit at $j$ and, at the dependents $j{+}1$ and $j{+}2$, to the unit of each of the top-8 tokens in the belief about the dependent, whether foresight uses the current or the predicted beliefs.

\paragraph{Special tokens.} The detector never counts special tokens (end of text, padding and the other reserved tokens), so they can never be green. We therefore tilt them as red tokens, with the boost $-\gamma\delta'$, and in beliefs they contribute the null rate $\gamma$. Crediting them as ordinary tokens instead makes foresight favor ending, because positions after an end-of-text token are predictable and, when the key makes them green, look valuable although they are never counted. This shortened the responses by 17--35 tokens on average in a 46-prompt screen. Leaving their probability untilted avoids this and keeps the share of very short answers at its unwatermarked level, and in an earlier configuration, tilting them as red lowered the mean log-perplexity further, by about 0.04 at $\delta=3$ and 4, without lowering the mean $z$-score.

\paragraph{Consistent order.} The predicted tilt of a position is computed only when that position wins the low-confidence choice. Each round computes it for the current winner, redraws every position's candidate with the same Gumbel noise, and recomputes the winner. The rounds stop when the winner already carries its predicted tilt, after at most four rounds per step (the cap was reached in fewer than one commit in 10{,}000). The committed token then differs from the winning draw in fewer than 0.01\% of the commits, against 2.5\% when positions are chosen with the current-belief tilt (the \emph{current-tilt order}).

\paragraph{Hyperparameters and cost.} Table~\ref{tab:hparams} lists the hyperparameters and where each comes from. On average, a step computes 2.1 position tilts and runs the predictor 0.72 times. Generation with our watermark takes 32.7\,s per 256-token response, against 21.9\,s with foresight from the current beliefs and the plain score (1.49 times, one A5000 GPU). Exact foresight, which re-runs the generator for all eight candidates, takes 47.1\,s (2.15 times) even when it is applied only at decision states (17\% of the commits). The exact-foresight runs of Appendix~\ref{app:contrasts} re-run the generator wherever the predictor is used, a superset of these commits. The predictor adds 3.5\,GB of bf16 weights.

\begin{table}[h]
\centering
\small
\setlength{\tabcolsep}{3pt}
\begin{tabular}{@{}lll@{}}
\toprule
Hyperparameter & Value & Choice \\
\midrule
Bands $B$ ($\gamma=1/B$) & 4 & prior work \\
Context width & 2 tokens & screened \\
Candidates tilted per position & top-50 & prior work \\
Beliefs per context position & top-4 & cost \\
Beliefs per dependent & top-8 & cost \\
Candidates with foresight & top-8 & cost \\
Foresight weight $\lambda$ & $\frac12$ & fixed \\
Predictor gate & 0.02 & cost \\
Rounds of the consistent order & 4 & fixed \\
Partition strength $\beta$ & 4 & $\delta'/T$ at $\delta=4$ \\
Temperature $T$ & 0.5 & prior work \\
\bottomrule
\end{tabular}
\caption{Hyperparameters of our watermark. Prior work: the value of \citet{gloaguen2026watermarking}. Screened: chosen among a few values (Appendix~\ref{app:ablations}). Cost: set to bound the cost. Fixed: fixed in advance.}
\label{tab:hparams}
\end{table}

\subsection{Detection and baselines}
\paragraph{Detection.} The text is re-tokenized with the LLaDA tokenizer. Position $i$ uses the context $(x_{i-2},x_{i-1})$, with a padding token for the first two positions, and each distinct (context, token) unit is counted once. The null distribution of the green count is the convolution, over distinct contexts, of the uniform distribution on the context's band counts (Proposition~\ref{prop:null}). Its upper tail at the observed count is the exact $p$-value, and $z=(k-\mathbb{E}k)/\mathrm{sd}(k)$ under the same distribution. Counting only the first occurrence of each context would also give an exact binomial test but would discard over a tenth of the units.

\paragraph{Baselines.} We run DLM-KGW \citep{gloaguen2026watermarking} with the official code in its WaterBench configuration, with $\gamma=0.25$, a one-token context with the sum hash, top-50 boosting and their predictive term. Their own detector scores the generated tokens. We run dgMARK with the official code and without a private key in two versions, sampling from the top-3 tokens (``top-3'', its main configuration) and greedy decoding (``greedy''), both with beam size one and temperature 0.5. At this temperature, LLaDA's sampler already draws every token at random, so sampling from the top-3 tokens adds few tokens of the favored parity (a match ratio of 0.666 against 0.660 with greedy decoding and 0.50 without a watermark), and the two versions detect about as many responses. With greedy decoding at temperature 0 instead, dgMARK detects only 61.8\% of the responses at 1\% FPR, at a log-perplexity of 1.394, because 23\% of its unwatermarked responses then repeat text, which raises its calibrated thresholds. Its three-beam search is 2.6 times slower and was not run on all prompts. Detection uses the official match-ratio statistic. The dgMARK paper evaluates continuations of news articles and stories and keeps only responses longer than 200 tokens, whereas we keep every response of every method. All methods use the same prompts, generation length, block length and temperature.

\FloatBarrier
\section{The foresight predictor}
\label{app:learned}
This appendix describes the training data, architecture and training of the foresight predictor, its fidelity on held-out decision states, and examples of post-placement beliefs.

\paragraph{Data.} We decode 1{,}612 distinct training prompts once each with the watermark in the low-confidence order, at one of four strengths ($\delta'=1.0$, 1.5, 2.0 and 2.5 for 412, 407, 410 and 383 prompts), with exact foresight and the band network of Appendix~\ref{app:partition} as the partition. A \emph{decision state} is a commit whose second most likely tilted candidate has probability at least 0.1. At every decision state, we place each of the top-8 candidates and re-run the generator. In a window of 13 positions around $j$, we record the inputs of the generator's top eight transformer layers and its final states, the candidates and their log-probabilities, and the exact post-placement top-32 beliefs of every masked position in the window. With one commit per step, the next forward pass is the post-placement pass of the committed token, so every second commit is also recorded with its one committed candidate at no extra cost. The checkpoint we use was trained on the states of these 1{,}612 decodings (59{,}425 decision-state records).

These training decodings differ from deployment in three ways. They use exact foresight instead of the predictor and the band network instead of the learned bands as the partition, and they record only decision states (second most likely tilted candidate of probability at least 0.1), whereas deployment uses the predictor wherever this probability is at least 0.02. The predictor's inputs and targets, the generator's states and post-placement beliefs, depend neither on the partition nor on how foresight was computed, so these differences only change which states are sampled. The predictor's fidelity carries over to the states of the learned bands. Evaluated on 1{,}329 decision states from decoding 48 of the evaluation prompts at $\delta=3$ with the learned bands and exact foresight, the predictor recovers the top-1 token in 60\% of the cases, explains 47\% of the variance of exact foresight and ranks the exact best candidate first in 54\% of the states, against 25\%, $-$5\% and 27\% for the current beliefs.

\paragraph{Architecture.} Given the inputs $h_{j-6},\dots,h_{j+6}$ of the generator's top eight transformer layers $\mathcal{B}$ in the current forward pass, an adapter $a_\theta$ rewrites the input at position $j$ for candidate $v$, $h'_j=h_j+a_\theta(h_j,E[v],\log p(v))$, with $E[v]$ the generator's input embedding of $v$. Fine-tuned copies $\mathcal{B}_\theta$ of the eight layers are run on the 13 positions with and without the rewrite, and the belief about a neighbor $j{+}d$ is predicted as
\begin{equation}
\label{eq:predictor}
\hat \pi_{j+d}(\cdot\mid v)=\mathrm{softmax}\big(W[r_{j+d}+\Delta_{j+d}]/T\big),
\end{equation}
where $\Delta=\mathrm{LN}(\mathcal{B}_\theta(h'))-\mathrm{LN}(\mathcal{B}_\theta(h))$, $r$ are the generator's final states, and $\mathrm{LN}$ and $W$ its final normalization and output projection. With $a_\theta=0$ the prediction equals the current belief.

\paragraph{Training.} Table~\ref{tab:predtrain} gives the training setup. The loss is the cross-entropy to the exact top-32 post-placement beliefs and their remaining mass, at temperature $T$, summed over the masked positions of the window (weight 1 for $j{-}1$, $j{+}1$ and $j{+}2$, which foresight reads, and 0.3 for the others). On the 2{,}002 validation records of these decodings, the predicted top-1 token of the neighbors $j{-}1$, $j{+}1$, $j{+}2$ agrees with the exact post-placement one in 55.9\% of the cases after one epoch and 61.2\% after five (cross-entropy 1.99 and 1.68 nats), against 25.5\% (4.51 nats) for the current beliefs. We use the checkpoint after the fourth epoch (60.8\%), the latest when our evaluation runs started. Fine-tuning only the top two of the eight layers reached 56.7\% after three epochs, and a fine-tuned copy of the last layer alone, 47.4\%.

\begin{table}[h]
\centering
\small
\setlength{\tabcolsep}{3pt}
\begin{tabular}{@{}p{0.3\columnwidth}p{0.66\columnwidth}@{}}
\toprule
Trainable parameters & 1.745B (8 layers) + 12.6M (adapter) \\
Adapter & 2-layer MLP, width 1{,}024, zero-initialized output \\
Optimizer & AdamW, $(\beta_1,\beta_2)=(0.9,0.98)$, weight decay 0.01 \\
Learning rates & $2\cdot10^{-5}$ (layers), $3\cdot10^{-4}$ (adapter) \\
Schedule & 300 warm-up steps, cosine, 6 epochs planned \\
Batch & 12 decision states (8 candidates each) or 96 committed-candidate states \\
Gradient clipping & 1.0 \\
Held-out prompts & 5\% (chosen by a hash of the prompt) \\
Hardware, time & one A6000 GPU, about 1 hour per epoch \\
\bottomrule
\end{tabular}
\caption{Training of the foresight predictor.}
\label{tab:predtrain}
\end{table}

\IfFileExists{tables/table_fidelity.tex}{\begin{table}[t]
\centering
\small
\setlength{\tabcolsep}{3.5pt}
\begin{tabular}{@{}lrrrr@{}}
\toprule
Foresight & Top-1 & Corr. & $R^2$ & Argmax \\
\midrule
Current beliefs & 26\% & 0.39 & $-$2\% & 28\% \\
Predictor & \textbf{61\%} & \textbf{0.73} & \textbf{49\%} & \textbf{54\%} \\
\textit{Top-1 token known} & -- & 0.91 & 68\% & 55\% \\
\bottomrule
\end{tabular}
\caption{Fidelity of approximate foresight on 2{,}397 held-out decision states of LLaDA-8B-Instruct in its low-confidence order, with the top-8 candidates of each state. Top-1 is the share of cases in which a masked neighbor's most likely token under the approximate beliefs is its exact post-placement one. Corr.\ and $R^2$ compare the approximate with the exact foresight $\Phi(v)$, centered within each state, by Pearson correlation and by the share of variance explained without refitting. Argmax is the share of states in which the approximation ranks the exact best candidate first. \textit{Top-1 token known} is a reference given each neighbor's exact post-placement top-1 token. Bold marks the best approximation in each column.}
\label{tab:fidelity}
\end{table}
}{}
\paragraph{Fidelity.} On 2{,}397 held-out decision states, those of the 5\% held-out prompts of the predictor's training decodings at which exact foresight varies across the candidates, Table~\ref{tab:fidelity} compares the foresight $\Phi(v)$ of the eight candidates computed from the current, the predicted and the exact beliefs, with the learned bands and centered within each state. We compute $R^2$ as $1-\sum(\hat\Phi-\Phi)^2/\sum(\Phi-\bar\Phi)^2$ without refitting, so a constant estimate scores 0 and a misleading one scores below 0.

\begin{table*}[t]
\centering
\small
\setlength{\tabcolsep}{2.5pt}
\begin{tabular}{@{}llll@{}}
\toprule
Context of $j$ & At $j$ & Belief about $j{+}1$ & Belief about $j{+}2$ \\
\midrule
\ldots have received investment from \M{} & \textit{Before placing} & \textit{wide}~0.27, \textit{with}~0.27 & \textit{of}~0.50, \textit{range}~0.30 \\
 & Place \textit{individuals} (0.59) & \textit{and}~0.43, \textit{with}~0.43 & \textit{small}~0.96, \textit{institutions}~0.02 \\
 & Place \textit{a} (0.36) & \textit{diverse}~0.79, \textit{range}~0.12 & \textit{of}~0.53, \textit{range}~0.47 \\
\midrule
\ldots typically made in private \M{} & \textit{Before placing} & \textit{before}~0.42, \textit{that}~0.22 & \textit{have}~0.23, \textit{they}~0.21 \\
 & Place \textit{companies} (0.79) & \textit{that}~0.73, \textit{before}~0.24 & \textit{have}~0.61, \textit{they}~0.32 \\
 & Place \textit{equity} (0.20) & \textit{funds}~0.90, \textit{or}~0.09 & \textit{,}~0.85, \textit{venture}~0.11 \\
\midrule
\ldots preceding ones . The \M{} & \textit{Before placing} & \textit{sequence}~0.50, \textit{common}~0.16 & \textit{Fibonacci}~0.36, \textit{numbers}~0.15 \\
 & Place \textit{sequence} (0.65) & \textit{starts}~0.82, \textit{typically}~0.13 & \textit{with}~0.81, \textit{starts}~0.06 \\
 & Place \textit{most} (0.16) & \textit{commonly}~0.73, \textit{common}~0.27 & \textit{used}~0.79, \textit{Fibonacci}~0.20 \\
\midrule
\ldots \M{} you want in \M{} & \textit{Before placing} & \textit{Fibonacci}~0.54, \textit{sequence}~0.42 & \textit{.}~0.92, \textit{(}~0.07 \\
 & Place \textit{the} (0.50) & \textit{sequence}~0.75, \textit{Fibonacci}~0.19 & \textit{.}~0.95, \textit{(}~0.04 \\
 & Place \textit{your} (0.50) & \textit{profit}~0.93, \textit{Fibonacci}~0.05 & \textit{.}~0.82, \textit{target}~0.13 \\
\midrule
\ldots from a donor . \M{} & \textit{Before placing} & \textit{diseased}~0.91, \textit{is}~0.05 & \textit{is}~0.35, \textit{kidney}~0.27 \\
 & Place \textit{This} (0.59) & \textit{is}~0.86, \textit{procedure}~0.06 & \textit{is}~0.74, \textit{because}~0.21 \\
 & Place \textit{The} (0.40) & \textit{diseased}~0.87, \textit{reason}~0.07 & \textit{kidney}~0.56, \textit{for}~0.27 \\
\midrule
\ldots 5 -point extension , \M{} & \textit{Before placing} & \textit{would}~0.60, \textit{find}~0.20 & \textit{find}~0.58, \textit{look}~0.19 \\
 & Place \textit{you} (0.78) & \textit{would}~0.96, \textit{find}~0.03 & \textit{find}~0.52, \textit{look}~0.15 \\
 & Place \textit{the} (0.22) & \textit{Fibonacci}~0.64, \textit{first}~0.09 & \textit{number}~0.53, \textit{numbers}~0.19 \\
\bottomrule
\end{tabular}
\caption{Real decision states of LLaDA-8B-Instruct in its low-confidence order, on prompts of the training pool. Each state gives the context of the committed position $j$ and the top-2 beliefs about the masked positions $j{+}1$ and $j{+}2$, before placement and after placing each of the two most likely candidates, whose probabilities at $j$ are in parentheses.}
\label{tab:examples}
\end{table*}

\paragraph{Examples.} Table~\ref{tab:examples} shows decision states of LLaDA-8B-Instruct in which the two most likely candidates lead the model to different neighbors. Before placement, the belief about $j{+}1$ is one distribution for all candidates. Because masked positions are predicted independently, it is often a mixture of the continuations of the different candidates (\emph{wide}/\emph{with} after ``investment from'') and sometimes a copy of the candidate itself (\emph{sequence} after ``The''). After placement, each candidate determines a different, usually sharper prediction. For example, \emph{individuals} is followed by \emph{and}/\emph{with}, and \emph{a} by \emph{diverse}. Foresight computed before placement cannot see this difference.

\paragraph{Design notes.} We build the predictor from the generator's top layers for three reasons. These layers already map the generator's hidden states to beliefs, so the predictor only has to learn how placing a candidate changes them, and its zero-initialized adapter makes it start from the current beliefs. It also reuses the inputs of these layers from the current forward pass, so the candidates of a commit share one pass through eight layers instead of a pass through the whole generator each. Finally, it fits exact foresight better than a predictor that does not reuse the generator's layers. In an earlier comparison, with both trained on the same decision states and scored on 1{,}098 held-out ones, it explained 25\% of the variance of exact foresight, against 18\% for a window transformer over the final hidden states, the architecture we use for Dream (Appendix~\ref{app:codeqa}), and both recovered the neighbors' post-placement top-1 token in 42\% of the cases. Two other alternatives were less attractive. First, re-running only the last $m$ layers of the generator, with the input at $j$ set to its exact post-placement value, recovers the neighbors' post-placement top-1 token in 28, 32, 44 and 64\% of the cases for $m=1,2,4,8$ (current beliefs: 26\%). The placed token thus changes the neighbors' predictions through the whole depth of the network, and even $m=8$ needs the exact input to these layers, that is, a pass through the first 24 layers per candidate. The predictor instead learns this input with its adapter. Second, sharpening the predicted beliefs (temperature 0.35 or 0.25 instead of 0.5) lowered the gain at equal KL in an experiment with an earlier configuration. Foresight averages the green indicator over the predicted beliefs, and a sharpened prediction commits to one neighbor and loses more often when it is wrong.

\section{Null distribution and calibration over keys}
\label{app:keys}
This appendix covers the comparison with DLM-KGW over five keys, the false-positive rates on text without a watermark, and why we calibrate thresholds over 100 keys.
\paragraph{Comparison over keys.} Table~\ref{tab:mkey} repeats the comparison with DLM-KGW under five generation keys (our main key 42 and keys 1 to 4) on 300 of the WaterBench prompts, at $\delta=2$ and 3. With the thresholds of the main text, pooled over 100 keys, our watermark exceeds theirs by 12.7 and 4.5 points in TPR@1\% and by 11.9 and 10.3 points in TPR@0.1\%. With each key's own thresholds, the differences are 9.9 and 5.1 points in TPR@1\% and 13.4 and 13.2 points in TPR@0.1\%. Over the same keys and prompts, our mean log-perplexity is 0.02 and 0.05 lower than theirs at $\delta=2$ and 3 (neither significant), and it varies less across keys (1.31--1.33 and 1.38--1.41, against 1.31--1.47 and 1.39--1.60).
\begin{table}[t]
\centering
\small
\setlength{\tabcolsep}{2.2pt}
\begin{tabular}{@{}llrrr@{}}
\toprule
Threshold & $\delta$ & Ours & DLM-KGW & Difference \\
\midrule
\multicolumn{5}{@{}l}{\emph{TPR@1\% (\%)}} \\
Key-pooled & 2 & 65.2 & 52.5 & \textbf{$+$12.7} {\footnotesize\color{black!60}[2.9, 20.9]} \\
 & 3 & 89.9 & 85.4 & \textbf{$+$4.5} {\footnotesize\color{black!60}[1.3, 7.9]} \\
\addlinespace[2pt]
Own-key & 2 & 59.5 & 49.5 & $+$9.9 {\footnotesize\color{black!60}[$-$1.1, 20.7]} \\
 & 3 & 88.5 & 83.4 & \textbf{$+$5.1} {\footnotesize\color{black!60}[0.9, 9.3]} \\
\midrule
\multicolumn{5}{@{}l}{\emph{TPR@0.1\% (\%)}} \\
Key-pooled & 2 & 40.8 & 28.9 & \textbf{$+$11.9} {\footnotesize\color{black!60}[2.3, 20.0]} \\
 & 3 & 80.1 & 69.7 & \textbf{$+$10.3} {\footnotesize\color{black!60}[5.0, 16.1]} \\
\addlinespace[2pt]
Own-key & 2 & 40.9 & 27.5 & \textbf{$+$13.4} {\footnotesize\color{black!60}[1.3, 25.9]} \\
 & 3 & 79.7 & 66.5 & \textbf{$+$13.2} {\footnotesize\color{black!60}[3.9, 22.6]} \\
\midrule
\multicolumn{5}{@{}l}{\emph{log-PPL}} \\
 & 2 & 1.325 & 1.349 & $-$0.024 {\footnotesize\color{black!60}[$-$0.112, 0.026]} \\
 & 3 & 1.402 & 1.452 & $-$0.051 {\footnotesize\color{black!60}[$-$0.156, 0.016]} \\
\bottomrule
\end{tabular}
\caption{TPR@1\%, TPR@0.1\% and log-PPL averaged over five generation keys, on 300 of the WaterBench-600 prompts with LLaDA-8B-Instruct in its low-confidence order. Key-pooled thresholds are the calibrated thresholds of the other tables, and own-key thresholds are calibrated on the unwatermarked responses re-detected with each key alone. The difference is ours minus DLM-KGW, with a 95\% CI from a bootstrap over prompts (paired) and keys, and bold marks the differences whose CI excludes zero.}
\label{tab:mkey}
\end{table}

\begin{table*}[t]
\centering
\small
\setlength{\tabcolsep}{3pt}
\begin{tabular}{@{}lrrrrrrr@{}}
\toprule
 & & \multicolumn{2}{c}{DLM-KGW} & \multicolumn{2}{c}{Unkeyed learned list} & \multicolumn{2}{c}{Learned bands} \\
 & & \multicolumn{2}{c}{(previous-token hash)} & \multicolumn{2}{c}{(our ablation)} & \multicolumn{2}{c}{(ours)} \\
\cmidrule(lr){3-4}\cmidrule(lr){5-6}\cmidrule(lr){7-8}
Text & $n$ & Green (\%) & FPR (\%) & Green (\%) & FPR (\%) & Green (\%) & FPR (\%) \\
\midrule
\multicolumn{8}{@{}l}{\emph{Human-written text}} \\
ELI5 & 532 & 24.9 & 0.6 & 31.0 & 30.1\rlap{$^\dagger$} & 25.3 & 0.9 \\
FiQA & 200 & 25.1 & 1.0 & 28.6 & 15.5\rlap{$^\dagger$} & 25.1 & 1.0 \\
C4 news & 284 & 25.2 & 0.4 & 28.4 & 3.5\rlap{$^\dagger$} & 25.1 & 1.8 \\
CNN/DailyMail & 1{,}000 & 25.4 & 2.1\rlap{$^\dagger$} & 28.6 & 13.5\rlap{$^\dagger$} & 25.2 & 0.8 \\
\midrule
\multicolumn{8}{@{}l}{\emph{LLaDA-8B-Instruct responses}} \\
Low-confidence & 581 & 24.3 & 0.2 & 26.2 & 5.5\rlap{$^\dagger$} & 25.4 & 1.2 \\
Random & 567 & 24.4 & 0.4 & 26.2 & 4.6\rlap{$^\dagger$} & 25.2 & 0.7 \\
\bottomrule
\end{tabular}
\caption{Green rate and false-positive rate (FPR) of three kinds of green lists on human-written text and on unwatermarked LLaDA-8B-Instruct responses to WaterBench-600, with one key and $\gamma=25\%$. The green rate is the share of scored units that are green, and the FPR is that of the exact test at a nominal 1\% level (Proposition~\ref{prop:null} for learned bands). Texts are scored on their first 256 tokens, and $n$ counts those with at least 50 scored units. The unkeyed learned list is a network that maps the context to a green list without a key, trained to keep the generator's tokens green at rate $\gamma$. A dagger marks an FPR significantly above 1\% (exact Clopper--Pearson 95\% CI).}
\label{tab:nullfull}
\end{table*}

\paragraph{False-positive rates.} Table~\ref{tab:nullfull} gives the false-positive rates and the green rates behind them. The unkeyed learned list makes 28--31\% of the units of human text green instead of $\gamma=25\%$, so every nominal $p$-value is too small. Learned bands stay within half a point of 25\% on every corpus, and the previous-token hash of DLM-KGW within 0.7 points.

\paragraph{Calibration over keys.} Over a random key, the null distribution is exact both for learned bands (Proposition~\ref{prop:null}) and for the binomial test of DLM-KGW, which also counts each distinct unit once. Because the count is discrete, an exact test rejects at most at its nominal level. For a fixed key and a fixed corpus, this no longer holds exactly, because texts share units such as frequent bigrams, and whether the key makes these units green or red moves the scores of many texts in the same direction. We re-detected the 600 unwatermarked low-confidence responses with 100 keys. For the official detector of DLM-KGW, the empirical 1\% threshold ranges over keys from 1.75 to 2.98, and that of the key with which we ran their code (the \emph{official key}, 2.01) lies in the lowest fifth. For learned bands, whose two-token contexts recur less, it ranges from 2.12 to 2.93 (our key: 2.50). Pooled over the keys, the thresholds are 2.33 for the official detector and 2.37 for learned bands, close to the 2.33 of a standard normal variable.

A key-specific shift would cancel in a single-key calibration if watermarked text inherited it, but it does not. Under the official key, the 202 units that occur in at least 5\% of the unwatermarked responses contribute $-0.17$ to the mean $z$-score of these responses and $+0.18$ ($\delta=3$) and $+0.30$ ($\delta=4$) to that of the watermarked responses, because the watermark steers the generator away from those units when they are red. A single-key calibration therefore depends strongly on the key, and we calibrate on the pooled scores. In the random order, the official key is typical for the official detector (threshold 2.31, against 2.37 pooled over keys). With each detector's own nominal $p$-value threshold instead (Table~\ref{tab:mainfull}), the differences between ours and DLM-KGW at $\delta=2$ and 3 are larger (21.2 and 5.7 points in nominal TPR@1\%). dgMARK has no key, so its thresholds come from its 600 unwatermarked scores alone. At 0.1\% FPR, the threshold is essentially their maximum.

\section{Additional WaterBench results}
\label{app:more}
\begin{table*}[t]
\centering
\small
\setlength{\tabcolsep}{5pt}
\begin{tabular}{llrrrrrrr}
\toprule
 & & \multicolumn{4}{c}{TPR (\%) $\uparrow$} & \multicolumn{2}{c}{log-PPL $\downarrow$} & Judge \\
\cmidrule(lr){3-6}\cmidrule(lr){7-8}
 & & \multicolumn{2}{c}{Nominal} & \multicolumn{2}{c}{Calibrated} & & & \\
\cmidrule(lr){3-4}\cmidrule(lr){5-6}
Method & $\delta$ & @1\% & @0.1\% & @1\% & @0.1\% & Mean & Median & $\uparrow$ \\
\midrule
No watermark & -- & 0.2 & 0.0 & -- & -- & 1.297 & 1.273 & 9.12 \\
dgMARK, greedy & -- & 90.8 & 83.3 & 92.7 & 84.2 & 1.550 & 1.531 & 9.07 \\
dgMARK, top-3 & -- & 93.0 & 83.8 & 92.3 & 84.3 & 1.557 & 1.547 & 9.12 \\
\midrule
DLM-KGW & 2 & 45.8 & 20.5 & 51.3 & 24.8 & \textbf{1.316} & \textbf{1.297} & 9.13 \\
\Method{} & 2 & \textbf{67.0} & \textbf{40.3} & \textbf{69.5} & \textbf{44.0} & 1.326 & 1.328 & \textbf{9.29} \\
\midrule
DLM-KGW & 3 & 82.8 & 66.7 & 85.7 & 70.0 & \textbf{1.387} & 1.409 & 8.97 \\
\Method{} & 3 & \textbf{88.5} & \textbf{78.0} & \textbf{89.3} & \textbf{79.3} & 1.398 & \textbf{1.388} & \textbf{9.30} \\
\midrule
DLM-KGW & 4 & 94.8 & 89.0 & 96.0 & 91.3 & 1.506 & 1.541 & 8.62 \\
\Method{} & 4 & \textbf{96.0} & \textbf{92.0} & \textbf{96.2} & \textbf{92.7} & \textbf{1.469} & \textbf{1.465} & \textbf{9.32} \\
\midrule
DLM-KGW & 5 & 97.3 & 95.0 & 97.7 & 95.3 & 1.602 & 1.650 & 8.15 \\
\Method{} & 5 & \textbf{97.7} & \textbf{95.8} & \textbf{98.3} & \textbf{96.0} & \textbf{1.575} & \textbf{1.592} & \textbf{9.26} \\
\bottomrule
\end{tabular}
\caption{Full results on WaterBench-600 with LLaDA-8B-Instruct in its low-confidence order, with nominal and calibrated TPRs (\S\ref{sec:setup}) and the mean and median log-PPL. In the No watermark row, the nominal columns give the false-positive rate. Bold marks the better of DLM-KGW and ours at the same $\delta$ without a significance test, and Table~\ref{tab:paired} gives the paired differences.}
\label{tab:mainfull}
\end{table*}

\begin{table*}[t]
\centering
\small
\setlength{\tabcolsep}{5pt}
\begin{tabular}{@{}lcrrrrrr@{}}
\toprule
 & & \multicolumn{4}{c}{Same $\delta$} & \multicolumn{2}{c}{Equal log-PPL} \\
\cmidrule(lr){3-6}\cmidrule(lr){7-8}
 & & \multicolumn{2}{c}{$\Delta$TPR (points)} & & & \multicolumn{2}{c}{$\Delta$TPR (points)} \\
\cmidrule(lr){3-4}\cmidrule(lr){7-8}
Order & $\delta$ & @1\% & @0.1\% & $\Delta$log-PPL & $\Delta$Judge & @1\% & @0.1\% \\
\midrule
Low-confidence & 2 & \textbf{$+$18.2} & \textbf{$+$19.2} & $+$0.010 & \textbf{$+$0.16} & \textbf{$+$13.6} & $+$13.1 \\
 & & {\footnotesize\color{black!60}[13.2, 23.2]} & {\footnotesize\color{black!60}[14.2, 24.3]} & {\footnotesize\color{black!60}[$-$0.014, 0.033]} & {\footnotesize\color{black!60}[0.01, 0.31]} & {\footnotesize\color{black!60}[2.3, 19.6]} & {\footnotesize\color{black!60}[$-$1.8, 20.4]} \\
 & 3 & \textbf{$+$3.7} & \textbf{$+$9.3} & $+$0.011 & \textbf{$+$0.33} & $+$2.7 & \textbf{$+$7.4} \\
 & & {\footnotesize\color{black!60}[0.5, 6.8]} & {\footnotesize\color{black!60}[5.2, 13.8]} & {\footnotesize\color{black!60}[$-$0.015, 0.038]} & {\footnotesize\color{black!60}[0.18, 0.49]} & {\footnotesize\color{black!60}[$-$1.0, 11.2]} & {\footnotesize\color{black!60}[1.3, 18.5]} \\
 & 4 & $+$0.2 & $+$1.3 & \textbf{$-$0.037} & \textbf{$+$0.70} & \textbf{$+$3.4} & \textbf{$+$8.0} \\
 & & {\footnotesize\color{black!60}[$-$1.8, 2.2]} & {\footnotesize\color{black!60}[$-$1.2, 4.0]} & {\footnotesize\color{black!60}[$-$0.071, $-$0.007]} & {\footnotesize\color{black!60}[0.52, 0.89]} & {\footnotesize\color{black!60}[0.3, 6.2]} & {\footnotesize\color{black!60}[1.8, 13.0]} \\
 & 5 & $+$0.7 & $+$0.7 & $-$0.027 & \textbf{$+$1.11} & $+$1.1 & $+$1.8 \\
 & & {\footnotesize\color{black!60}[$-$0.7, 2.0]} & {\footnotesize\color{black!60}[$-$1.2, 2.7]} & {\footnotesize\color{black!60}[$-$0.063, 0.008]} & {\footnotesize\color{black!60}[0.91, 1.33]} & {\footnotesize\color{black!60}[$-$0.2, 2.5]} & {\footnotesize\color{black!60}[$-$0.1, 4.0]} \\
\midrule
Random & 3 & 0.0 & \textbf{$+$7.0} & $-$0.036 & -- & -- & -- \\
 & & {\footnotesize\color{black!60}[$-$3.7, 3.7]} & {\footnotesize\color{black!60}[2.3, 12.2]} & {\footnotesize\color{black!60}[$-$0.078, 0.010]} &  &  &  \\
 & 4 & $-$2.3 & $-$1.8 & \textbf{$-$0.140} & -- & \textbf{$+$6.6} & \textbf{$+$20.3} \\
 & & {\footnotesize\color{black!60}[$-$4.7, 0.0]} & {\footnotesize\color{black!60}[$-$4.8, 1.3]} & {\footnotesize\color{black!60}[$-$0.174, $-$0.106]} &  & {\footnotesize\color{black!60}[2.8, 9.4]} & {\footnotesize\color{black!60}[13.6, 24.6]} \\
\bottomrule
\end{tabular}
\caption{Paired differences of ours minus DLM-KGW on WaterBench-600 with LLaDA-8B-Instruct in its low-confidence and random orders, at the same $\delta$ and at equal log-PPL (\S\ref{sec:setup}). Brackets give paired bootstrap 95\% CIs, and bold marks the estimates whose CI excludes zero. A dash means that our mean log-PPL lies below the range of DLM-KGW (equal log-PPL) or that the judge was not run (random order).}
\label{tab:paired}
\end{table*}

Table~\ref{tab:paired} gives the paired differences behind Table~\ref{tab:main}, at the same strength and at equal log-perplexity, for both decoding orders. Table~\ref{tab:mainfull} extends Table~\ref{tab:main} with the nominal TPR@1\% and TPR@0.1\% and the median log-perplexity, which is less sensitive than the mean to the few very short answers whose per-token log-perplexity can be extreme.

\paragraph{Loops.} Table~\ref{tab:loopfree} separates the effect of repetition loops. Loops lower both the evidence and the log-perplexity of a response. In a looping response, more than a third of the tokens repeat an earlier unit and are not counted (about an eighth in the other responses, with two-token units). For DLM-KGW, looping responses therefore have a mean $z$-score 0.4--1.1 lower than that of the other responses ($\delta=2$--5). Looping responses also have a log-perplexity 0.2--0.5 lower for every method, so loops lower the mean log-perplexity of DLM-KGW. The judge rates looping responses 7.2--7.8 for ours, 5.3--7.0 for theirs and 7.0 without a watermark. With their loop rule (more than half of the bigrams repeated), the share of looping DLM-KGW responses still grows with $\delta$, from 4.8 to 11.7\% (3.7\% without a watermark), while ours stays at 1.5--2.0\%. In the random order, the main setting of their paper, responses rarely loop for any method (2.8--4.7\% with our rule), consistent with their observation that the watermark does not increase repetitiveness.

\begin{table*}[t]
\centering
\small
\setlength{\tabcolsep}{4.2pt}
\begin{tabular}{@{}crrrrrrrr@{}}
\toprule
 & \multicolumn{2}{c}{Rep.} & \multicolumn{2}{c}{$\Delta$Judge vs.\ No watermark} & \multicolumn{2}{c}{$\Delta$Judge vs.\ DLM-KGW} & \multicolumn{2}{c}{$\Delta$log-PPL vs.\ DLM-KGW} \\
\cmidrule(lr){2-3}\cmidrule(lr){4-5}\cmidrule(lr){6-7}\cmidrule(lr){8-9}
$\delta$ & DLM-KGW & Ours & All & No loop & All & No loop & All & No loop \\
\midrule
2 & 19.3 & 18.2 & \textbf{$+$0.17} & $+$0.01 & \textbf{$+$0.16} & $-$0.06 & $+$0.010 & $-$0.008 \\
 & & & {\footnotesize\color{black!60}[0.04, 0.31]} & {\footnotesize\color{black!60}[$-$0.08, 0.11]} & {\footnotesize\color{black!60}[0.01, 0.31]} & {\footnotesize\color{black!60}[$-$0.14, 0.02]} & {\footnotesize\color{black!60}[$-$0.014, 0.033]} & {\footnotesize\color{black!60}[$-$0.035, 0.019]} \\
3 & 23.3 & 14.0 & \textbf{$+$0.18} & $+$0.03 & \textbf{$+$0.33} & $+$0.03 & $+$0.011 & $-$0.007 \\
 & & & {\footnotesize\color{black!60}[0.05, 0.32]} & {\footnotesize\color{black!60}[$-$0.06, 0.12]} & {\footnotesize\color{black!60}[0.18, 0.49]} & {\footnotesize\color{black!60}[$-$0.06, 0.13]} & {\footnotesize\color{black!60}[$-$0.015, 0.038]} & {\footnotesize\color{black!60}[$-$0.039, 0.024]} \\
4 & 26.3 & 14.3 & \textbf{$+$0.21} & 0.00 & \textbf{$+$0.70} & $+$0.08 & \textbf{$-$0.037} & \textbf{$-$0.097} \\
 & & & {\footnotesize\color{black!60}[0.07, 0.35]} & {\footnotesize\color{black!60}[$-$0.10, 0.10]} & {\footnotesize\color{black!60}[0.52, 0.89]} & {\footnotesize\color{black!60}[$-$0.03, 0.19]} & {\footnotesize\color{black!60}[$-$0.071, $-$0.007]} & {\footnotesize\color{black!60}[$-$0.131, $-$0.065]} \\
5 & 30.5 & 12.2 & \textbf{$+$0.15} & $-$0.04 & \textbf{$+$1.11} & \textbf{$+$0.18} & $-$0.027 & \textbf{$-$0.139} \\
 & & & {\footnotesize\color{black!60}[0.01, 0.28]} & {\footnotesize\color{black!60}[$-$0.12, 0.05]} & {\footnotesize\color{black!60}[0.91, 1.33]} & {\footnotesize\color{black!60}[0.06, 0.31]} & {\footnotesize\color{black!60}[$-$0.063, 0.008]} & {\footnotesize\color{black!60}[$-$0.174, $-$0.102]} \\
\bottomrule
\end{tabular}
\caption{Repetition loops and paired quality differences on WaterBench-600 with LLaDA-8B-Instruct in its low-confidence order. Without a watermark, Rep.\ is 19.3\%. Each $\Delta$ is ours minus the compared method, over all 600 prompts (All) or over the prompts on which neither compared response loops (No loop, 385--446 prompts). Brackets give paired bootstrap 95\% CIs, and bold marks the estimates whose CI excludes zero.}
\label{tab:loopfree}
\end{table*}

\IfFileExists{tables/table_judges.tex}{\begin{table}[t]
\centering
\small
\setlength{\tabcolsep}{3.4pt}
\begin{tabular}{@{}lcrrrr@{}}
\toprule
 & & \multicolumn{2}{c}{Quality $\uparrow$} & \multicolumn{2}{c}{Similarity $\uparrow$} \\
\cmidrule(lr){3-4}\cmidrule(l){5-6}
Method & $\delta$ & Gemma & Qwen & Gemma & Qwen \\
\midrule
No watermark & -- & 9.12 & 8.55 & -- & -- \\
dgMARK, greedy & -- & 9.07 & 8.49 & 6.40 & 7.00 \\
dgMARK, top-3 & -- & 9.12 & 8.47 & 6.34 & 6.95 \\
\midrule
DLM-KGW & 2 & 9.13 & 8.53 & 6.63 & 7.23 \\
\Method{} & 2 & 9.29 & 8.65 & 6.57 & 7.24 \\
\midrule
DLM-KGW & 3 & 8.97 & 8.39 & 6.41 & 7.00 \\
\Method{} & 3 & 9.30 & 8.66 & 6.42 & 7.12 \\
\midrule
DLM-KGW & 4 & 8.62 & 8.10 & 6.17 & 6.80 \\
\Method{} & 4 & 9.32 & 8.59 & 6.27 & 6.97 \\
\midrule
DLM-KGW & 5 & 8.15 & 7.76 & 5.93 & 6.55 \\
\Method{} & 5 & 9.26 & 8.59 & 6.07 & 6.89 \\
\bottomrule
\end{tabular}
\caption{Quality and similarity of the responses of Table~\ref{tab:main} (WaterBench-600, LLaDA-8B-Instruct, low-confidence order) under two judges, Gemma-4-12B-it and Qwen3-32B (FP8 weights), both greedy. Quality: mean of style and consistency (1--10) with the rubric of \citet{gloaguen2026watermarking} and the same instruction text for every method. Similarity: shared meaning with the unwatermarked response to the same prompt (0--10). Unwatermarked responses of another sampler reach 6.71 and 6.77 (Gemma) and 7.33 and 7.32 (Qwen).}
\label{tab:judges}
\end{table}
}{}
\paragraph{Second judge and similarity.} Table~\ref{tab:judges} re-scores the responses of Table~\ref{tab:main} with a second judge, Qwen3-32B \citep{yang2025qwen3} with FP8 weights, and also rates how much meaning each response shares with the unwatermarked response to the same prompt, from 0 (no shared meaning) to 10 (equivalent meaning). The two judges agree closely. Over the 6{,}600 responses, their quality scores correlate at 0.84, and over the 11 methods, their means correlate at 0.99 (similarity: 0.85 over 7{,}200 pairs and 0.96 over 12 sets). Qwen3-32B also rates ours above DLM-KGW at every strength, by 0.12, 0.27, 0.49 and 0.83 at $\delta=2$ to 5 (all significant), and at least as high as text without a watermark. Unwatermarked responses of another sampler share 6.7 (Gemma) and 7.3 (Qwen) with the main unwatermarked responses, the level that sampling alone reaches. Both watermarks come close to it at $\delta=2$ and fall below it as $\delta$ grows, and ours stays closer to the unwatermarked responses than DLM-KGW at $\delta=4$ and 5 under Qwen3-32B (by 0.17 and 0.33, both significant) and under Gemma-4-12B-it (by 0.10 and 0.14, not significant).

\IfFileExists{tables/table_tasks.tex}{\begin{table}[t]
\centering
\small
\setlength{\tabcolsep}{2.5pt}
\begin{tabular}{@{}lcrrrrrr@{}}
\toprule
 & & \multicolumn{2}{c}{ELI5} & \multicolumn{2}{c}{FiQA} & \multicolumn{2}{c}{AlpacaFarm} \\
\cmidrule(lr){3-4}\cmidrule(lr){5-6}\cmidrule(lr){7-8}
 & & TPR &  & TPR &  & TPR &  \\
Method & $\delta$ & @1\% & Len. & @1\% & Len. & @1\% & Len. \\
\midrule
No watermark & -- & -- & 232 & -- & 219 & -- & 214 \\
dgMARK, top-3 & -- & 99.5 & 219 & 95.5 & 204 & 82.0 & 209 \\
\midrule
DLM-KGW & 2 & 59.0 & 229 & 48.5 & 227 & 46.5 & 217 \\
\Method{} & 2 & \textbf{77.5} & 230 & \textbf{77.0} & 224 & \textbf{54.0} & 214 \\
\midrule
DLM-KGW & 3 & 93.0 & 234 & 86.5 & 230 & \textbf{77.5} & 221 \\
\Method{} & 3 & \textbf{98.0} & 232 & \textbf{96.5} & 226 & 73.5 & 214 \\
\midrule
DLM-KGW & 4 & 99.5 & 233 & 99.0 & 235 & \textbf{89.5} & 228 \\
\Method{} & 4 & 99.5 & 237 & \textbf{100.0} & 226 & 89.0 & 216 \\
\midrule
DLM-KGW & 5 & 100.0 & 231 & 100.0 & 239 & 93.0 & 226 \\
\Method{} & 5 & 100.0 & 243 & 100.0 & 224 & \textbf{95.0} & 220 \\
\bottomrule
\end{tabular}
\caption{Results per task on WaterBench-600 (200 prompts each) with LLaDA-8B-Instruct in its low-confidence order. TPR@1\% uses the calibrated thresholds of Table~\ref{tab:main}, and Len.\ is the mean response length in tokens. On AlpacaFarm, the share of responses shorter than 20 tokens is 4.5\% without a watermark, 0.5--2.0\% for DLM-KGW, 2.0--4.0\% for ours and 6.5\% for dgMARK, and no ELI5 or FiQA response is that short. Bold marks the better of DLM-KGW and ours at the same $\delta$.}
\label{tab:tasks}
\end{table}
}{}
\paragraph{Per task.} Table~\ref{tab:tasks} splits the results by task. Ours detects more responses than DLM-KGW on ELI5 and FiQA at $\delta=2$ and 3, and as many at $\delta=4$ and 5, where both detect 99--100\%. AlpacaFarm is the hardest task for both, partly because some of its instructions have short answers and many have answers with little freedom. There, neither method detects consistently more. DLM-KGW lengthens AlpacaFarm responses at higher strengths, which makes them easier to detect, whereas ours keeps their length.

\paragraph{Random order.} In LLaDA's random order, with the predictor trained on low-confidence decodings, ours matches DLM-KGW in TPR@1\% at $\delta=3$ and exceeds it by 7.0 points in TPR@0.1\%. At $\delta=4$, it detects slightly fewer responses (not significant). Its log-perplexity is lower by 0.04 (not significant) and 0.14 at $\delta=3$ and 4 (Table~\ref{tab:paired}). At equal log-perplexity, ours detects 6.6 and 20.3 points more in TPR@1\% and TPR@0.1\% at $\delta=4$. At $\delta=3$, our log-perplexity lies below their range, so the two cannot be compared at equal log-perplexity. Thresholds come from the 600 unwatermarked random-order responses, pooled over 100 keys for both keyed detectors.

\section{Paired contrasts of the crossed ablation}
\label{app:contrasts}
\IfFileExists{tables/table_facfull.tex}{\begin{table*}[t]
\centering
\small
\setlength{\tabcolsep}{5pt}
\begin{tabular}{@{}lllrrrrrr@{}}
\toprule
 & & & $\bar z$ & \multicolumn{2}{c}{TPR (\%) $\uparrow$} & log-PPL & Judge & Rep. \\
\cmidrule(lr){5-6}
$\delta$ & Partition & Foresight & $\uparrow$ & @1\% & @0.1\% & $\downarrow$ & $\uparrow$ & $\downarrow$ \\
\midrule
2 & Keyed hash & Current beliefs & 2.35 & 49.0 & 24.5 & 1.337 & 9.18 & 18.2 \\
 &  & Predictor & 2.61 & 59.5 & 31.7 & 1.351 & \textbf{9.29} & \textbf{17.2} \\
 & Learned bands & None ($\lambda=0$) & 2.06 & 39.3 & 15.3 & 1.317 & 9.22 & 16.3 \\
 &  & Current beliefs & 2.77 & 63.7 & 40.0 & \textbf{1.308} & 9.24 & 19.2 \\
\rowcolor{black!7} &  & Predictor & \textbf{2.95} & \textbf{69.5} & \textbf{44.0} & 1.326 & \textbf{9.29} & 18.2 \\
 &  & Exact$^\dagger$ & 3.10 & 73.2 & 47.7 & 1.326 & -- & 17.8 \\
\midrule
3 & Keyed hash & Current beliefs & 3.78 & 84.5 & 65.8 & 1.397 & \textbf{9.35} & 15.2 \\
 &  & Predictor & 3.98 & 86.5 & 71.2 & 1.403 & 9.31 & 15.5 \\
 & Learned bands & None ($\lambda=0$) & 3.11 & 70.7 & 49.2 & 1.372 & 9.30 & 16.2 \\
 &  & Current beliefs & 4.25 & 88.7 & 76.8 & \textbf{1.391} & 9.33 & \textbf{12.8} \\
\rowcolor{black!7} &  & Predictor & \textbf{4.35} & \textbf{89.3} & \textbf{79.3} & 1.398 & 9.30 & 14.0 \\
 &  & Exact$^\dagger$ & 4.66 & 93.3 & 83.8 & 1.378 & 9.27 & 16.8 \\
\midrule
4 & Keyed hash & Current beliefs & 5.27 & 94.5 & 90.0 & 1.507 & \textbf{9.40} & \textbf{10.7} \\
 &  & Predictor & 5.45 & 95.7 & 91.5 & 1.489 & 9.39 & 11.3 \\
 & Learned bands & None ($\lambda=0$) & 4.15 & 89.5 & 76.0 & 1.451 & 9.22 & 14.8 \\
 &  & Current beliefs & 5.69 & \textbf{96.3} & \textbf{92.8} & 1.482 & 9.25 & 12.7 \\
\rowcolor{black!7} &  & Predictor & \textbf{5.97} & 96.2 & 92.7 & \textbf{1.469} & 9.32 & 14.3 \\
 &  & Exact$^\dagger$ & 6.11 & 96.8 & 93.5 & 1.469 & 9.31 & 14.5 \\
\bottomrule
\end{tabular}
\caption{All metrics of the crossed ablation of Table~\ref{tab:fac}. TPR@1\% and TPR@0.1\% use calibrated thresholds. Bold marks the best of the four crossed cells at each $\delta$, and the shaded rows are ours.}
\label{tab:facfull}
\end{table*}
}{}
\IfFileExists{tables/table_contrasts.tex}{\begin{table*}[p]
\centering
\footnotesize
\setlength{\tabcolsep}{3pt}
\renewcommand{\arraystretch}{0.88}
\begin{tabular}{@{}p{2.8cm}p{2.2cm}crrrrr@{}}
\toprule
 & & & & \multicolumn{2}{c}{$\Delta$TPR (points)} & & \\
\cmidrule(lr){5-6}
Contrast & Held fixed & $\delta$ & $\Delta\bar z$ & @1\% & @0.1\% & $\Delta$log-PPL & $\Delta$Judge \\
\midrule
\multicolumn{8}{@{}l}{\emph{Partition}} \\
\multirow{6}{=}{Learned bands vs.\ keyed hash} & \multirow{6}{=}{Current beliefs} & 2 & \textbf{$+$0.41} & \textbf{$+$14.7} & \textbf{$+$15.5} & \textbf{$-$0.028} & $+$0.06 \\
 & & & {\scriptsize\color{black!60}[0.29, 0.54]} & {\scriptsize\color{black!60}[9.5, 19.8]} & {\scriptsize\color{black!60}[10.5, 20.5]} & {\scriptsize\color{black!60}[$-$0.055, $-$0.002]} & {\scriptsize\color{black!60}[$-$0.08, 0.20]} \\
 &  & 3 & \textbf{$+$0.47} & \textbf{$+$4.2} & \textbf{$+$11.0} & $-$0.006 & $-$0.02 \\
 & & & {\scriptsize\color{black!60}[0.33, 0.61]} & {\scriptsize\color{black!60}[0.8, 7.5]} & {\scriptsize\color{black!60}[6.7, 15.5]} & {\scriptsize\color{black!60}[$-$0.033, 0.020]} & {\scriptsize\color{black!60}[$-$0.15, 0.10]} \\
 &  & 4 & \textbf{$+$0.42} & $+$1.8 & \textbf{$+$2.8} & $-$0.025 & \textbf{$-$0.15} \\
 & & & {\scriptsize\color{black!60}[0.27, 0.56]} & {\scriptsize\color{black!60}[0.0, 3.7]} & {\scriptsize\color{black!60}[0.3, 5.3]} & {\scriptsize\color{black!60}[$-$0.054, 0.003]} & {\scriptsize\color{black!60}[$-$0.28, $-$0.02]} \\
\addlinespace[2pt]
\multirow{6}{=}{Learned bands vs.\ keyed hash} & \multirow{6}{=}{Predictor} & 2 & \textbf{$+$0.34} & \textbf{$+$10.0} & \textbf{$+$12.3} & $-$0.026 & $-$0.01 \\
 & & & {\scriptsize\color{black!60}[0.21, 0.47]} & {\scriptsize\color{black!60}[5.0, 15.0]} & {\scriptsize\color{black!60}[7.2, 17.7]} & {\scriptsize\color{black!60}[$-$0.055, 0.002]} & {\scriptsize\color{black!60}[$-$0.14, 0.12]} \\
 &  & 3 & \textbf{$+$0.37} & $+$2.8 & \textbf{$+$8.2} & $-$0.004 & 0.00 \\
 & & & {\scriptsize\color{black!60}[0.24, 0.51]} & {\scriptsize\color{black!60}[0.0, 5.7]} & {\scriptsize\color{black!60}[4.0, 12.3]} & {\scriptsize\color{black!60}[$-$0.030, 0.022]} & {\scriptsize\color{black!60}[$-$0.13, 0.12]} \\
 &  & 4 & \textbf{$+$0.52} & $+$0.5 & $+$1.2 & $-$0.020 & $-$0.07 \\
 & & & {\scriptsize\color{black!60}[0.38, 0.66]} & {\scriptsize\color{black!60}[$-$1.3, 2.2]} & {\scriptsize\color{black!60}[$-$1.0, 3.3]} & {\scriptsize\color{black!60}[$-$0.046, 0.007]} & {\scriptsize\color{black!60}[$-$0.20, 0.05]} \\
\midrule
\multicolumn{8}{@{}l}{\emph{Foresight}} \\
\multirow{6}{=}{Predictor vs.\ current beliefs} & \multirow{6}{=}{Learned bands} & 2 & \textbf{$+$0.18} & \textbf{$+$5.8} & $+$4.0 & $+$0.017 & $+$0.05 \\
 & & & {\scriptsize\color{black!60}[0.07, 0.28]} & {\scriptsize\color{black!60}[1.5, 10.2]} & {\scriptsize\color{black!60}[$-$0.7, 8.5]} & {\scriptsize\color{black!60}[$-$0.002, 0.037]} & {\scriptsize\color{black!60}[$-$0.07, 0.17]} \\
 &  & 3 & $+$0.10 & $+$0.7 & $+$2.5 & $+$0.007 & $-$0.02 \\
 & & & {\scriptsize\color{black!60}[$-$0.03, 0.21]} & {\scriptsize\color{black!60}[$-$2.0, 3.5]} & {\scriptsize\color{black!60}[$-$1.0, 6.2]} & {\scriptsize\color{black!60}[$-$0.014, 0.028]} & {\scriptsize\color{black!60}[$-$0.14, 0.09]} \\
 &  & 4 & \textbf{$+$0.28} & $-$0.2 & $-$0.2 & $-$0.013 & $+$0.08 \\
 & & & {\scriptsize\color{black!60}[0.15, 0.40]} & {\scriptsize\color{black!60}[$-$1.5, 1.2]} & {\scriptsize\color{black!60}[$-$2.2, 1.8]} & {\scriptsize\color{black!60}[$-$0.036, 0.011]} & {\scriptsize\color{black!60}[$-$0.04, 0.20]} \\
\addlinespace[2pt]
\multirow{6}{=}{Predictor vs.\ current beliefs} & \multirow{6}{=}{Keyed hash} & 2 & \textbf{$+$0.25} & \textbf{$+$10.5} & \textbf{$+$7.2} & $+$0.015 & $+$0.11 \\
 & & & {\scriptsize\color{black!60}[0.15, 0.35]} & {\scriptsize\color{black!60}[6.2, 15.0]} & {\scriptsize\color{black!60}[3.0, 11.5]} & {\scriptsize\color{black!60}[$-$0.004, 0.034]} & {\scriptsize\color{black!60}[0.00, 0.23]} \\
 &  & 3 & \textbf{$+$0.20} & $+$2.0 & \textbf{$+$5.3} & $+$0.006 & $-$0.04 \\
 & & & {\scriptsize\color{black!60}[0.09, 0.31]} & {\scriptsize\color{black!60}[$-$1.3, 4.8]} & {\scriptsize\color{black!60}[1.3, 9.7]} & {\scriptsize\color{black!60}[$-$0.018, 0.030]} & {\scriptsize\color{black!60}[$-$0.16, 0.07]} \\
 &  & 4 & \textbf{$+$0.18} & $+$1.2 & $+$1.5 & $-$0.018 & $-$0.01 \\
 & & & {\scriptsize\color{black!60}[0.06, 0.31]} & {\scriptsize\color{black!60}[$-$0.5, 2.8]} & {\scriptsize\color{black!60}[$-$0.8, 3.8]} & {\scriptsize\color{black!60}[$-$0.043, 0.005]} & {\scriptsize\color{black!60}[$-$0.11, 0.11]} \\
\addlinespace[2pt]
\multirow{6}{=}{Predictor vs.\ none ($\lambda=0$)} & \multirow{6}{=}{Learned bands} & 2 & \textbf{$+$0.88} & \textbf{$+$30.2} & \textbf{$+$28.7} & $+$0.009 & $+$0.07 \\
 & & & {\scriptsize\color{black!60}[0.78, 0.99]} & {\scriptsize\color{black!60}[25.2, 35.0]} & {\scriptsize\color{black!60}[24.2, 33.2]} & {\scriptsize\color{black!60}[$-$0.012, 0.029]} & {\scriptsize\color{black!60}[$-$0.06, 0.18]} \\
 &  & 3 & \textbf{$+$1.25} & \textbf{$+$18.7} & \textbf{$+$30.2} & \textbf{$+$0.026} & $+$0.01 \\
 & & & {\scriptsize\color{black!60}[1.13, 1.36]} & {\scriptsize\color{black!60}[15.0, 22.2]} & {\scriptsize\color{black!60}[26.2, 34.5]} & {\scriptsize\color{black!60}[0.001, 0.049]} & {\scriptsize\color{black!60}[$-$0.11, 0.13]} \\
 &  & 4 & \textbf{$+$1.82} & \textbf{$+$6.7} & \textbf{$+$16.7} & $+$0.018 & $+$0.10 \\
 & & & {\scriptsize\color{black!60}[1.69, 1.96]} & {\scriptsize\color{black!60}[4.5, 9.0]} & {\scriptsize\color{black!60}[13.5, 20.2]} & {\scriptsize\color{black!60}[$-$0.006, 0.042]} & {\scriptsize\color{black!60}[$-$0.03, 0.23]} \\
\addlinespace[2pt]
\multirow{6}{=}{Current beliefs vs.\ none ($\lambda=0$)} & \multirow{6}{=}{Learned bands} & 2 & \textbf{$+$0.71} & \textbf{$+$24.3} & \textbf{$+$24.7} & $-$0.009 & $+$0.02 \\
 & & & {\scriptsize\color{black!60}[0.60, 0.81]} & {\scriptsize\color{black!60}[19.5, 29.0]} & {\scriptsize\color{black!60}[20.3, 29.2]} & {\scriptsize\color{black!60}[$-$0.031, 0.014]} & {\scriptsize\color{black!60}[$-$0.10, 0.14]} \\
 &  & 3 & \textbf{$+$1.15} & \textbf{$+$18.0} & \textbf{$+$27.7} & $+$0.019 & $+$0.04 \\
 & & & {\scriptsize\color{black!60}[1.03, 1.27]} & {\scriptsize\color{black!60}[14.2, 21.7]} & {\scriptsize\color{black!60}[23.2, 32.2]} & {\scriptsize\color{black!60}[$-$0.003, 0.042]} & {\scriptsize\color{black!60}[$-$0.08, 0.16]} \\
 &  & 4 & \textbf{$+$1.54} & \textbf{$+$6.8} & \textbf{$+$16.8} & \textbf{$+$0.031} & $+$0.02 \\
 & & & {\scriptsize\color{black!60}[1.42, 1.67]} & {\scriptsize\color{black!60}[4.7, 9.2]} & {\scriptsize\color{black!60}[13.5, 20.2]} & {\scriptsize\color{black!60}[0.005, 0.057]} & {\scriptsize\color{black!60}[$-$0.11, 0.15]} \\
\midrule
\multicolumn{8}{@{}l}{\emph{Both components}} \\
\multirow{6}{=}{Learned bands + predictor vs.\ keyed hash + current beliefs} &  & 2 & \textbf{$+$0.59} & \textbf{$+$20.5} & \textbf{$+$19.5} & $-$0.011 & $+$0.11 \\
 & & & {\scriptsize\color{black!60}[0.46, 0.72]} & {\scriptsize\color{black!60}[15.3, 25.7]} & {\scriptsize\color{black!60}[14.5, 24.7]} & {\scriptsize\color{black!60}[$-$0.038, 0.015]} & {\scriptsize\color{black!60}[$-$0.02, 0.24]} \\
 &  & 3 & \textbf{$+$0.57} & \textbf{$+$4.8} & \textbf{$+$13.5} & $+$0.001 & $-$0.04 \\
 & & & {\scriptsize\color{black!60}[0.43, 0.70]} & {\scriptsize\color{black!60}[1.7, 8.2]} & {\scriptsize\color{black!60}[9.3, 17.7]} & {\scriptsize\color{black!60}[$-$0.023, 0.026]} & {\scriptsize\color{black!60}[$-$0.16, 0.07]} \\
 &  & 4 & \textbf{$+$0.70} & $+$1.7 & \textbf{$+$2.7} & \textbf{$-$0.038} & $-$0.07 \\
 & & & {\scriptsize\color{black!60}[0.55, 0.85]} & {\scriptsize\color{black!60}[$-$0.2, 3.7]} & {\scriptsize\color{black!60}[0.2, 5.2]} & {\scriptsize\color{black!60}[$-$0.066, $-$0.011]} & {\scriptsize\color{black!60}[$-$0.20, 0.05]} \\
\midrule
\multicolumn{8}{@{}l}{\emph{Decoding choices}} \\
\multirow{4}{=}{Unit-aware vs.\ plain score} & \multirow{4}{=}{Learned bands, predictor} & 3 & $+$0.04 & $-$0.8 & $+$1.3 & \textbf{$+$0.033} & \textbf{$+$0.35} \\
 & & & {\scriptsize\color{black!60}[$-$0.06, 0.13]} & {\scriptsize\color{black!60}[$-$2.8, 1.2]} & {\scriptsize\color{black!60}[$-$1.5, 4.3]} & {\scriptsize\color{black!60}[0.015, 0.052]} & {\scriptsize\color{black!60}[0.23, 0.48]} \\
 &  & 4 & \textbf{$+$0.16} & $+$0.2 & $+$1.2 & \textbf{$+$0.032} & \textbf{$+$0.44} \\
 & & & {\scriptsize\color{black!60}[0.06, 0.26]} & {\scriptsize\color{black!60}[$-$1.0, 1.3]} & {\scriptsize\color{black!60}[$-$0.7, 3.2]} & {\scriptsize\color{black!60}[0.011, 0.053]} & {\scriptsize\color{black!60}[0.29, 0.59]} \\
\addlinespace[2pt]
\multirow{4}{=}{Unit-aware vs.\ plain score} & \multirow{4}{=}{Learned bands, current beliefs} & 3 & $+$0.07 & $-$0.3 & $+$0.7 & \textbf{$+$0.026} & \textbf{$+$0.33} \\
 & & & {\scriptsize\color{black!60}[$-$0.03, 0.18]} & {\scriptsize\color{black!60}[$-$2.8, 2.3]} & {\scriptsize\color{black!60}[$-$3.2, 4.3]} & {\scriptsize\color{black!60}[0.007, 0.044]} & {\scriptsize\color{black!60}[0.20, 0.46]} \\
 &  & 4 & $+$0.04 & $+$0.2 & $+$1.7 & $+$0.018 & \textbf{$+$0.29} \\
 & & & {\scriptsize\color{black!60}[$-$0.07, 0.16]} & {\scriptsize\color{black!60}[$-$1.2, 1.7]} & {\scriptsize\color{black!60}[$-$0.3, 3.7]} & {\scriptsize\color{black!60}[$-$0.004, 0.039]} & {\scriptsize\color{black!60}[0.15, 0.44]} \\
\addlinespace[2pt]
\multirow{4}{=}{Unit-aware vs.\ plain score} & \multirow{4}{=}{Keyed hash, current beliefs} & 3 & \textbf{$+$0.10} & $+$1.0 & $+$1.0 & $+$0.006 & \textbf{$+$0.18} \\
 & & & {\scriptsize\color{black!60}[0.01, 0.19]} & {\scriptsize\color{black!60}[$-$1.7, 3.5]} & {\scriptsize\color{black!60}[$-$2.8, 4.8]} & {\scriptsize\color{black!60}[$-$0.012, 0.024]} & {\scriptsize\color{black!60}[0.06, 0.31]} \\
 &  & 4 & \textbf{$+$0.10} & $+$0.8 & $+$1.3 & \textbf{$+$0.029} & \textbf{$+$0.23} \\
 & & & {\scriptsize\color{black!60}[0.00, 0.19]} & {\scriptsize\color{black!60}[$-$0.3, 2.2]} & {\scriptsize\color{black!60}[$-$0.8, 3.7]} & {\scriptsize\color{black!60}[0.011, 0.046]} & {\scriptsize\color{black!60}[0.11, 0.35]} \\
\addlinespace[2pt]
\multirow{4}{=}{Consistent vs.\ current-tilt order} & \multirow{4}{=}{Learned bands, predictor, plain score} & 3 & \textbf{$-$0.17} & $-$2.0 & $-$2.3 & \textbf{$-$0.044} & \textbf{$-$0.16} \\
 & & & {\scriptsize\color{black!60}[$-$0.29, $-$0.04]} & {\scriptsize\color{black!60}[$-$4.5, 0.3]} & {\scriptsize\color{black!60}[$-$5.5, 1.0]} & {\scriptsize\color{black!60}[$-$0.065, $-$0.021]} & {\scriptsize\color{black!60}[$-$0.32, 0.00]} \\
 &  & 4 & $-$0.04 & 0.0 & $-$0.5 & \textbf{$-$0.055} & $-$0.10 \\
 & & & {\scriptsize\color{black!60}[$-$0.17, 0.08]} & {\scriptsize\color{black!60}[$-$1.5, 1.5]} & {\scriptsize\color{black!60}[$-$2.7, 1.5]} & {\scriptsize\color{black!60}[$-$0.078, $-$0.030]} & {\scriptsize\color{black!60}[$-$0.26, 0.08]} \\
\bottomrule
\end{tabular}
\caption{Paired differences behind \S\ref{sec:components} and Table~\ref{tab:decoding}, on WaterBench-600 with LLaDA-8B-Instruct in its low-confidence order. Each row gives the first configuration minus the second, with the settings under Held fixed shared by both. Brackets give paired bootstrap 95\% CIs, and bold marks the estimates whose CI excludes zero. A dash marks a difference that was not computed.}
\label{tab:contrasts}
\end{table*}
}{}
\IfFileExists{tables/table_decoding.tex}{\begin{table*}[t]
\centering
\small
\setlength{\tabcolsep}{3.4pt}
\begin{tabular}{@{}llrrrrrrrrrrrr@{}}
\toprule
 & & \multicolumn{6}{c}{$\delta=3$} & \multicolumn{6}{c}{$\delta=4$} \\
\cmidrule(lr){3-8}\cmidrule(lr){9-14}
 & & $\bar z$ & \multicolumn{2}{c}{TPR (\%) $\uparrow$} & log-PPL & Judge & Rep. & $\bar z$ & \multicolumn{2}{c}{TPR (\%) $\uparrow$} & log-PPL & Judge & Rep. \\
\cmidrule(lr){4-5}\cmidrule(lr){10-11}
Score & Order & $\uparrow$ & @1\% & @0.1\% & $\downarrow$ & $\uparrow$ & $\downarrow$ & $\uparrow$ & @1\% & @0.1\% & $\downarrow$ & $\uparrow$ & $\downarrow$ \\
\midrule
\rowcolor{black!7}Unit-aware & Consistent & 4.35 & 89.3 & 79.3 & 1.398 & 9.30 & 14.0 & 5.97 & 96.2 & 92.7 & 1.469 & 9.32 & 14.3 \\
Plain & Consistent & 4.32 & 90.2 & 78.0 & 1.365 & 8.96 & 21.0 & 5.82 & 96.0 & 91.5 & 1.437 & 8.88 & 20.8 \\
Plain & Current-tilt & 4.49 & 92.2 & 80.3 & 1.408 & 9.12 & 18.7 & 5.86 & 96.0 & 92.0 & 1.492 & 8.98 & 18.8 \\
\bottomrule
\end{tabular}
\caption{Decoding choices of our watermark with learned bands and the predictor, on WaterBench-600 with LLaDA-8B-Instruct in its low-confidence order. The plain score omits the unit-aware correction (\S\ref{sec:formulation}). The consistent order is that of \S\ref{sec:decoding}, and the current-tilt order chooses the position with the current-belief tilt and applies the predicted tilt only afterwards (Appendix~\ref{app:decoding}). Columns are as in Table~\ref{tab:fac}, with TPR@0.1\% and Rep.\ added, and the shaded row is ours.}
\label{tab:decoding}
\end{table*}
}{}
Table~\ref{tab:facfull} gives every metric of the crossed ablation, and Table~\ref{tab:contrasts} the paired differences behind \S\ref{sec:components}, including those of the two decoding choices, the unit-aware score and the consistent order. The unit-aware score raises the mean log-perplexity by about 0.03 because it reduces repetition loops, whose log-perplexity is low. On responses without loops, its effect on log-perplexity is at most 0.016 and not significant. The current-tilt order commits a token other than the draw that won the choice of position in 2.5\% of the commits. With the plain score, it raises the mean $z$-score and the judge score at $\delta=3$ by 0.17 and 0.16 relative to the consistent order, but also the log-perplexity of responses without loops (by 0.031 and 0.048 at $\delta=3$ and 4). With both choices, ours is rated 0.18 and 0.34 higher than the variant with the plain score and the current-tilt order (Table~\ref{tab:decoding}).

\paragraph{Additivity.} The interaction of the two components in mean $z$-score is the predictor's gain with learned bands minus its gain with the keyed hash, [(learned bands, predictor) $-$ (learned bands, current beliefs)] $-$ [(keyed hash, predictor) $-$ (keyed hash, current beliefs)]. It is $-0.07$, $-0.10$ and $+0.10$ at $\delta=2$, 3 and 4, none of them significant.

\begin{figure*}[t]
\centering
\includegraphics[width=\textwidth]{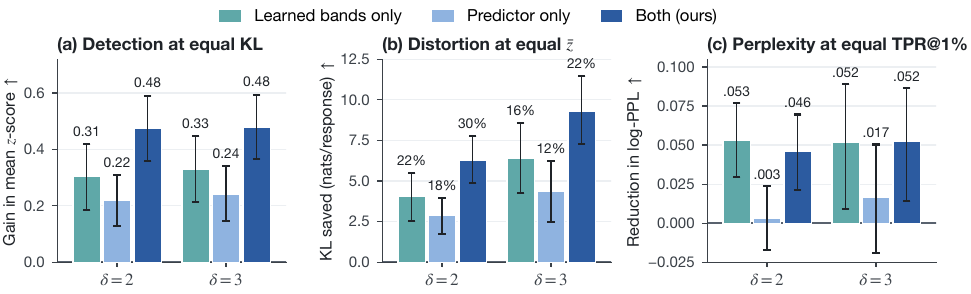}
\caption{The two learned components on the detection--quality trade-off, each against the cell with neither (keyed hash, current beliefs), on WaterBench-600 with LLaDA-8B-Instruct in its low-confidence order (unit-aware score and consistent order in every cell). (a) Gain in the mean $z$-score at equal KL divergence from the model's distribution. (b) KL saved at equal mean $z$-score, with the saving as a share of the reference cell's KL above each bar. (c) Reduction of the mean log-perplexity (Qwen3-8B) at equal TPR@1\%. The reference cell is interpolated linearly between its strengths $\delta=2$, 3 and 4, and error bars are paired bootstrap 95\% CIs.}
\label{fig:components}
\end{figure*}

\paragraph{Second judge.} Table~\ref{tab:facjudges} re-rates the crossed ablation in one run per judge with the same instructions for every cell. Neither component lowers the score of Qwen3-32B significantly at any strength, and the predictor with the keyed hash raises it at $\delta=2$ ($+$0.09). The only significant decrease under either judge is the lower Gemma-4-12B-it score of the learned partition with current-belief foresight at $\delta=4$ ($-$0.15), which Qwen3-32B puts at $-$0.08, not significantly. With the predictor, the learned partition changes the score at $\delta=4$ by $-$0.07 and 0.00, and the two components together change it by $-$0.07 to $+$0.11 under either judge, never significantly.
\IfFileExists{tables/table_facjudges.tex}{\begin{table*}[t]
\centering
\small
\setlength{\tabcolsep}{5pt}
\begin{tabular}{@{}llrrrrrr@{}}
\toprule
 & & \multicolumn{2}{c}{$\delta=2$} & \multicolumn{2}{c}{$\delta=3$} & \multicolumn{2}{c}{$\delta=4$} \\
\cmidrule(lr){3-4}\cmidrule(lr){5-6}\cmidrule(l){7-8}
Partition & Foresight & G & Q & G & Q & G & Q \\
\midrule
Keyed hash & Current beliefs & 9.18 & 8.58 & 9.35 & 8.69 & 9.40 & 8.67 \\
 & Predictor & 9.29 & 8.67 & 9.31 & 8.64 & 9.39 & 8.63 \\
Learned bands & Current beliefs & 9.24 & 8.60 & 9.33 & 8.69 & 9.24 & 8.59 \\
\rowcolor{black!6} & Predictor & 9.29 & 8.64 & 9.31 & 8.67 & 9.32 & 8.63 \\
\midrule
Learned bands $-$ keyed hash & current beliefs & $+$0.06 & $+$0.01 & $-$0.02 & 0.00 & \textbf{$-$0.15} & $-$0.08 \\
 & predictor & $-$0.01 & $-$0.02 & 0.00 & $+$0.03 & $-$0.07 & 0.00 \\
Predictor $-$ current beliefs & learned bands & $+$0.05 & $+$0.05 & $-$0.02 & $-$0.02 & $+$0.08 & $+$0.03 \\
 & keyed hash & $+$0.11 & \textbf{$+$0.09} & $-$0.04 & $-$0.05 & $-$0.01 & $-$0.05 \\
Both $-$ neither &  & $+$0.11 & $+$0.06 & $-$0.04 & $-$0.02 & $-$0.07 & $-$0.04 \\
\bottomrule
\end{tabular}
\caption{Judge scores of the crossed ablation of Table~\ref{tab:fac} (WaterBench-600, LLaDA-8B-Instruct, low-confidence order), re-rated in one run per judge with the same instruction text for every cell: Gemma-4-12B-it (G) and Qwen3-32B with FP8 weights (Q), both greedy, mean of style and consistency (1--10). The bottom rows give paired differences, with the setting after the label held fixed, and bold marks those whose paired bootstrap 95\% CI excludes zero. On these 7{,}200 responses, the two judges correlate at 0.78. Gemma scores can differ by 0.01 from Table~\ref{tab:fac} through rounding and, for some cells, another run.}
\label{tab:facjudges}
\end{table*}
}{}

\paragraph{Without the unit-aware score.} An earlier configuration of our sampler used the plain score and the current-tilt order. In the same crossed ablation with that configuration, the learned partition raised the mean $z$-score by 0.50 and 0.47 at $\delta=3$ and 4 with current-belief foresight, the predictor by 0.31 and 0.21 with learned bands, and the two together by 0.80 and 0.68, all significantly. The learned partition also lowered the judge score, by 0.17 and 0.22. The detection gains of the two components therefore do not depend on the unit-aware score.

\paragraph{Exact foresight.} Re-running the generator for each of the top-8 candidates at the commits where the predictor is used raises the mean $z$-score over the predictor by 0.15, 0.31 and 0.14 at $\delta=2$, 3 and 4 (all significant), at the same $\delta$ and about the same KL (Table~\ref{tab:exact}).
\begin{table}[h]
\centering
\small
\setlength{\tabcolsep}{4pt}
\begin{tabular}{@{}lcrrr@{}}
\toprule
 & & & \multicolumn{2}{c}{TPR (\%)} \\
\cmidrule(l){4-5}
Foresight & $\delta$ & $\bar z$ & @1\% & @0.1\% \\
\midrule
Predictor & 2 & 2.95 & 69.5 & 44.0 \\
Exact & 2 & 3.10 & 73.2 & 47.7 \\
\addlinespace[2pt]
Predictor & 3 & 4.35 & 89.3 & 79.3 \\
Exact & 3 & 4.66 & 93.3 & 83.8 \\
\addlinespace[2pt]
Predictor & 4 & 5.97 & 96.2 & 92.7 \\
Exact & 4 & 6.11 & 96.8 & 93.5 \\
\bottomrule
\end{tabular}
\caption{Predicted against exact foresight (learned partition, WaterBench-600, low-confidence order). TPR@1\% and TPR@0.1\% use calibrated thresholds.}
\label{tab:exact}
\end{table}

\paragraph{Foresight at equal KL.} Relative to the cell without foresight ($\lambda=0$), predicted foresight adds 0.94 to the mean $z$-score at equal KL at $\delta=3$, and current-belief foresight 0.80. At $\delta=4$, both use more KL than that cell at its strongest setting, and extrapolating it linearly gives 1.29 and 0.92.

\paragraph{Equal log-perplexity.} At equal mean log-perplexity instead of equal KL, the predictor's gain in mean $z$-score (learned bands) is significant at $\delta=4$ and 5 (0.49 and 0.37) but not at $\delta=2$ and 3. The log-perplexity of single runs varies by a few hundredths, which makes this comparison less precise than the one at equal KL. The same interpolation within the crossed ablation places both components on a better detection--quality frontier (Figure~\ref{fig:components}). At equal mean $z$-score, the learned partition, the predictor and the two together need 22, 18 and 30\% less KL than the cell with neither at $\delta=2$ (16, 12 and 22\% at $\delta=3$), all significantly. At equal TPR@1\%, the two together lower the mean log-perplexity by 0.05 at $\delta=2$ and 3, and the learned partition alone by 0.05 at both strengths, all significantly, whereas the predictor leaves it unchanged. At equal log-perplexity, the two together detect 4.7 and 5.1 points more at $\delta=3$ and 4, and at equal KL they raise the mean $z$-score by 0.48 at $\delta=2$ and 3, all significantly.

\paragraph{Random bands.} To separate learning from fixing each token's band, we also decode WaterBench-600 at $\delta=3$ with current-belief foresight and random bands, which assign each token to a band by a fixed hash of its id, without learning (thresholds calibrated with this partition). This run reaches a mean $z$-score of 3.86, against 3.78 for the keyed hash ($+$0.08, not significant) and 4.25 for the learned partition. Learning thus adds 0.39 to the mean $z$-score and 6.7 points to TPR@0.1\%, both significantly, with no significant change in log-perplexity ($-$0.015).

\paragraph{Further keys.} Table~\ref{tab:kabl} repeats the crossed ablation at $\delta=3$ with four further generation keys (keys 1--4) on the 300 prompts of the multi-key comparison. Averaged over the keys, the learned partition raises the mean $z$-score by 0.33 with current-belief foresight and by 0.35 with the predictor, and the predictor raises it by 0.15 with the keyed hash and by 0.17 with learned bands, all significantly. Together, the two components raise the mean $z$-score by 0.50, with no significant change in log-perplexity. Their interaction is 0.02 and not significant. The gain of the learned partition varies more across keys (0.11 to 0.62 with current-belief foresight) than that of the predictor (0.12 to 0.22).
\IfFileExists{tables/table_kabl.tex}{\begin{table*}[t]
\centering
\small
\setlength{\tabcolsep}{4pt}
\begin{tabular}{@{}llrrrrrr@{}}
\toprule
Partition & Foresight & $\bar z$ & TPR@1\% (\%) & log-PPL & & & \\
\midrule
Keyed hash & Current beliefs & 3.83 & 84.4 & 1.387 &  &  &  \\
Keyed hash & Predictor & 3.98 & 86.9 & 1.389 &  &  &  \\
Learned bands & Current beliefs & 4.16 & 88.2 & 1.399 &  &  &  \\
Learned bands & Predictor & 4.33 & 89.2 & 1.399 &  &  &  \\
\midrule
 & & \multicolumn{4}{c}{$\Delta\bar z$ per key} & \multicolumn{2}{c}{Mean over keys} \\
\cmidrule(lr){3-6}\cmidrule(lr){7-8}
\multicolumn{2}{@{}l}{Contrast} & Key 1 & Key 2 & Key 3 & Key 4 & $\Delta\bar z$ & $\Delta$log-PPL \\
\midrule
\multicolumn{8}{@{}l}{Learned bands $-$ keyed hash} \\
\multicolumn{2}{@{}l}{\quad Current beliefs} & \textbf{$+$0.62} & $+$0.11 & \textbf{$+$0.42} & $+$0.18 & \textbf{$+$0.33} {\footnotesize\color{black!60}[0.23, 0.43]} & $+$0.012 {\footnotesize\color{black!60}[$-$0.008, 0.032]} \\
\multicolumn{2}{@{}l}{\quad Predictor} & \textbf{$+$0.78} & $+$0.12 & \textbf{$+$0.47} & $+$0.05 & \textbf{$+$0.35} {\footnotesize\color{black!60}[0.26, 0.45]} & $+$0.010 {\footnotesize\color{black!60}[$-$0.007, 0.028]} \\
\multicolumn{8}{@{}l}{Predictor $-$ current beliefs} \\
\multicolumn{2}{@{}l}{\quad Keyed hash} & $+$0.03 & $+$0.15 & \textbf{$+$0.17} & \textbf{$+$0.25} & \textbf{$+$0.15} {\footnotesize\color{black!60}[0.07, 0.23]} & $+$0.002 {\footnotesize\color{black!60}[$-$0.014, 0.017]} \\
\multicolumn{2}{@{}l}{\quad Learned bands} & \textbf{$+$0.18} & \textbf{$+$0.16} & \textbf{$+$0.22} & $+$0.12 & \textbf{$+$0.17} {\footnotesize\color{black!60}[0.09, 0.25]} & 0.000 {\footnotesize\color{black!60}[$-$0.017, 0.016]} \\
\multicolumn{2}{@{}l}{Both $-$ neither} & \textbf{$+$0.81} & \textbf{$+$0.27} & \textbf{$+$0.63} & \textbf{$+$0.30} & \textbf{$+$0.50} {\footnotesize\color{black!60}[0.41, 0.60]} & $+$0.012 {\footnotesize\color{black!60}[$-$0.007, 0.031]} \\
\multicolumn{2}{@{}l}{Interaction} & $+$0.15 & $+$0.02 & $+$0.05 & $-$0.13 & $+$0.02 {\footnotesize\color{black!60}[$-$0.09, 0.13]} & $-$0.002 {\footnotesize\color{black!60}[$-$0.024, 0.020]} \\
\bottomrule
\end{tabular}
\caption{The crossed ablation of Table~\ref{tab:fac} under four further generation keys (keys 1, 2, 3 and 4) at $\delta=3$, on the same 300 of the WaterBench-600 prompts for every key, with LLaDA-8B-Instruct in its low-confidence order. The top rows average each cell over the keys, and their TPRs are nominal. The bottom rows give paired differences in $\bar z$ per key and, averaged over the keys, in $\bar z$ and log-PPL with paired bootstrap 95\% CIs. Indented rows give the level of the other component. Both $-$ neither compares learned bands with the predictor against the keyed hash with current beliefs, and the interaction is the predictor's gain with learned bands minus its gain with the keyed hash. Bold marks a difference whose 95\% CI excludes zero, for one key or over the keys.}
\label{tab:kabl}
\end{table*}
}{}

\section{Design ablations}
\label{app:ablations}
This appendix compares the context width, the number of bands, the foresight weight and the predictor gate, and measures what foresight beyond the two dependents could add.
\IfFileExists{tables/table_ctxh.tex}{\begin{table*}[t]
\centering
\small
\setlength{\tabcolsep}{3pt}
\begin{tabular}{@{}crrrrrrrrr@{}}
\toprule
 & & & \multicolumn{7}{c}{TPR@1\% (\%)} \\
\cmidrule(lr){4-10}
 & & & No attack & \multicolumn{2}{c}{Deletion} & \multicolumn{2}{c}{Insertion} & \multicolumn{2}{c}{Substitution} \\
\cmidrule(lr){5-6}\cmidrule(lr){7-8}\cmidrule(lr){9-10}
$h$ & $\bar z$ & log-PPL & & 10\% & 20\% & 10\% & 20\% & 10\% & 20\% \\
\midrule
1 & 3.40 & 1.518 & 77.3 & \textbf{65.3} & \textbf{53.3} & 64.0 & \textbf{56.7} & 56.7 & \textbf{47.3} \\
\rowcolor{black!7}2 & \textbf{3.78} & \textbf{1.494} & 81.3 & 64.7 & 44.7 & \textbf{66.7} & 51.3 & \textbf{64.0} & 31.3 \\
3 & 3.52 & 1.588 & 74.7 & 50.7 & 23.3 & 56.0 & 30.7 & 45.3 & 19.3 \\
4 & 3.66 & 1.601 & \textbf{84.0} & 46.0 & 12.0 & 50.0 & 28.7 & 44.0 & 12.0 \\
5 & 3.42 & 1.625 & 79.3 & 35.3 & 9.3 & 36.0 & 10.7 & 30.7 & 6.0 \\
6 & 3.37 & 1.606 & 71.3 & 27.3 & 8.7 & 33.3 & 14.0 & 22.0 & 5.3 \\
\bottomrule
\end{tabular}
\caption{Context width $h$ of learned bands at $\delta=4$, on the 150 development prompts (Table~\ref{tab:splits}) with LLaDA-8B-Instruct in its low-confidence order. All rows use the plain score and no foresight ($\lambda=0$). TPR@1\% is nominal, without an attack and after deleting, inserting or substituting 10 or 20\% of the tokens at random. The shaded row is the width we use ($h=2$), and bold marks the best value in each column.}
\label{tab:ctxh}
\end{table*}
}{}
\paragraph{Context width.} Learned bands work with any context width. For an $h$-token context $c_i=(x_{i-h},\dots,x_{i-1})$, we hash the first $h-1$ tokens into one integer and keep $x_{i-1}$ separately, so $h=2$ recovers our context, and Proposition~\ref{prop:null} holds for every $h$. Foresight for $h>2$ would have to consider $h$ dependents with $h-1$ other context positions each, so this comparison uses no foresight ($\lambda=0$). Table~\ref{tab:ctxh} gives one run of the 150 development prompts per width at $\delta=4$ with the same seeds. Two-token contexts give the highest mean $z$-score and the lowest log-perplexity. Wider contexts are much less robust, because an edited token changes the units of $h+1$ positions \citep[cf.][]{kirchenbauer2024reliability}, and one-token contexts are the most robust but detect less without edits. These robustness rates are nominal TPR@1\% without foresight, so they are not comparable with Table~\ref{tab:robust}.

\begin{table}[t]
\centering
\small
\setlength{\tabcolsep}{4.5pt}
\begin{tabular}{@{}rrrrrr@{}}
\toprule
 & \multicolumn{2}{c}{Per-token $z$ at $e^4$} & KL at $e^4$ & \multicolumn{2}{c}{At the KL of $B=4$} \\
\cmidrule(lr){2-3}\cmidrule(lr){5-6}
$B$ & Learned & Hash & (nats) & Per-token $z$ & $\beta$ \\
\midrule
2 & 0.408 & 0.377 & 0.358 & 0.448 & 4.66 \\
3 & 0.481 & 0.453 & 0.431 & 0.488 & 4.07 \\
\rowcolor{black!7}4 & 0.515 & 0.480 & 0.444 & 0.515 & 4.00 \\
6 & 0.531 & 0.493 & 0.418 & 0.548 & 4.12 \\
8 & 0.527 & 0.487 & 0.381 & 0.573 & 4.30 \\
12 & 0.501 & 0.459 & 0.315 & 0.608 & 4.64 \\
16 & 0.471 & 0.431 & 0.266 & 0.635 & 4.94 \\
\bottomrule
\end{tabular}
\caption{Offline comparison of the number of bands $B$ ($\gamma=1/B$) on the 10{,}277 held-out commit records from unwatermarked decoding of LLaDA-8B-Instruct in random order (Appendix~\ref{app:partition}). The per-token $z$ is $(g-\gamma)/\sqrt{\gamma(1-\gamma)}$, where $g$ is the expected probability, over a uniformly random key, that the committed token is green after a tilt of odds $e^{\beta}$ toward the green band. Learned bands are optimized at $\beta=4$, the Hash column uses a keyed hash, and KL is the tilt's KL divergence per token. The last two columns use, for each $B$, the $\beta$ at which learned bands have the KL of $B=4$. The shaded row is the $B$ we use.}
\label{tab:bands}
\end{table}

\paragraph{Number of bands.} In an offline measure on held-out records, the $z$-score gained per counted token at a fixed distortion grows with the number of bands (Table~\ref{tab:bands}). On 440 evaluation prompts decoded with four and with eight bands (plain score, foresight from the current beliefs, $\delta=3$ and 4, the same seeds), eight bands cost less KL at the same $\delta$ (46.3 against 64.3 nats per response at $\delta=4$). At equal KL, they reach a mean $z$-score 0.17 higher than four bands (significantly), but at equal mean log-perplexity 0.19 lower (not significantly). We keep $B=4$, the value of \citet{gloaguen2026watermarking}.

\begin{table}[t]
\centering
\small
\setlength{\tabcolsep}{3.1pt}
\begin{tabular}{@{}llrrrrr@{}}
\toprule
 & & $\bar z$ & \multicolumn{2}{c}{TPR (\%) $\uparrow$} & KL & log-PPL \\
\cmidrule(lr){4-5}
$\lambda$ & Foresight & $\uparrow$ & @1\% & @0.1\% & $\downarrow$ & $\downarrow$ \\
\midrule
$\frac14$ & Current beliefs & 3.75 & 83.3 & 60.7 & 28.9 & 1.365 \\
 & Predictor & 3.80 & 88.7 & 65.3 & 28.9 & 1.368 \\
\midrule
$\frac12$ & Current beliefs & 4.16 & 89.3 & 76.0 & 34.0 & 1.399 \\
\rowcolor{black!7} & Predictor & 4.42 & 90.0 & 82.0 & 34.9 & 1.363 \\
 & \quad gate 0 & 4.35 & 92.7 & 84.0 & 34.7 & 1.381 \\
\midrule
1 & Current beliefs & 5.47 & 96.7 & 94.7 & 57.3 & 1.487 \\
 & Predictor & 5.84 & 95.3 & 92.0 & 56.8 & 1.468 \\
 & \quad gate 0 & 5.96 & 96.7 & 93.3 & 58.6 & 1.466 \\
\bottomrule
\end{tabular}
\caption{Foresight weight $\lambda$ and source of foresight at $\delta=3$, on 150 of the WaterBench-600 prompts (learned bands, one run with the same seeds). TPRs use calibrated thresholds, and KL is the tilt's KL divergence summed over a response's commits. The predictor is used where the second most likely tilted candidate has probability at least 0.02, or at every position tilt in the rows marked gate 0. The shaded row is ours.}
\label{tab:sweep}
\end{table}

\paragraph{Foresight weight and predictor gate.} On development prompts at $\delta=4$ (one run, the same seeds), $\lambda=0.25$, 0.5 and 1 give a mean $z$-score of 5.01, 5.93 and 7.80 at a mean log-perplexity of 1.475, 1.477 and 1.653. From 0.25 to 0.5 the score rises at an unchanged log-perplexity, whereas beyond $\frac12$ the weight mainly sets the strength (from 0.5 to 1, the log-perplexity rises by 0.18 and the KL by 40 nats per response). Table~\ref{tab:sweep} crosses the weight with the source of foresight at $\delta=3$ on 150 WaterBench prompts. The predictor's gain in mean $z$-score over current beliefs grows with the weight, from 0.05 at $\lambda=\frac14$ (not significant) to 0.26 at $\frac12$ and 0.37 at 1, at about the same KL. On these prompts, its gain at $\lambda=\frac12$ is larger than on all 600 prompts (Table~\ref{tab:fac}). Using the predictor at every position tilt (gate 0) changes the mean $z$-score by $-$0.07 and 0.12 at $\lambda=\frac12$ and 1, neither significantly. Using it only at decision states (gate 0.1) did not change it significantly at $\delta=4$ either. The gate of 0.02 avoids most predictor calls.

\paragraph{Foresight beyond the dependents.} With a two-token context, a candidate enters exactly three units, the unit at $j$ and those at its two dependents. Placing it also changes the beliefs about farther positions, so farther units depend on it through the model alone. We measured how much this matters on 3{,}073 held-out decision states from a second recording of the predictor's training decodings, which stores the exact post-placement beliefs at offsets $-6$ to 6 for each of the top-8 candidates, with our partition and key. Across candidates, the probability-weighted standard deviation of a unit's expected greenness is 0.20 at $j$, 0.18 and 0.16 at the two dependents, 0.08, 0.06, 0.05 and 0.04 at offsets 3 to 6, and at most 0.05 to the left of $j$. Accordingly, extending the score from the two dependents to offsets 1 to 6 (with the same $\lambda=\frac12$) raises the expected evidence of the exact tilt by 12\% at $\delta=3$ and 11\% at $\delta=4$, but also its KL by 7\% and 6\%, so the evidence per unit of KL grows by only about 4\%. Adding the two dependents to the own term, in contrast, raises it by about 30\%. A longer horizon is thus a possible but small improvement, and it would require predicting the beliefs at farther positions.

\section{Code, question answering and Dream}
\label{app:codeqa}
This appendix details the code and question-answering evaluation and the experiments on Dream.
\IfFileExists{tables/table_codeqa.tex}{\begin{table*}[t]
\centering
\small
\setlength{\tabcolsep}{3pt}
\begin{tabular}{lrrrrrrrr}
\toprule
 & \multicolumn{2}{c}{HumanEval} & \multicolumn{2}{c}{MBPP} & \multicolumn{2}{c}{GPQA-Diamond} & \multicolumn{2}{c}{MMLU-Pro} \\
\cmidrule(lr){2-3}\cmidrule(lr){4-5}\cmidrule(lr){6-7}\cmidrule(lr){8-9}
Method & Score & TPR@1\% (\%) & Score & TPR@1\% (\%) & Score & TPR@1\% (\%) & Score & TPR@1\% (\%) \\
\midrule
No watermark & 35.4 & -- & 37.0 & -- & 20.7 & -- & 40.0 & -- \\
DLM-KGW & 34.8 & \textbf{17.7} & 32.5 & 15.0 & \textbf{26.8} & \textbf{72.2} & 35.0 & 60.0 \\
dgMARK, top-3 & 35.4 & 6.1 & 33.5 & 3.5 & 19.2 & 36.4 & \textbf{42.1} & 27.1 \\
\Method{} & \textbf{39.6} & 12.8 & \textbf{38.0} & \textbf{25.0} & 23.7 & 69.2 & 32.1 & \textbf{62.9} \\
\bottomrule
\end{tabular}
\caption{Code and question answering per task with LLaDA-8B-Instruct in its low-confidence order at $\delta=4$, with our components learned on WaterBench-style prompts, without code or QA training prompts. Score is pass@1 for code and accuracy for question answering, both in \%. TPR@1\% uses thresholds calibrated on the unwatermarked responses to the same prompts. Bold marks the best watermark in each column without a significance test.}
\label{tab:codeqa}
\end{table*}
}{}
\paragraph{Protocol.} With LLaDA-8B-Instruct, zero-shot, we decode all 164 HumanEval problems \citep{chen2021humaneval}, a fixed random subset of 200 of the 500 MBPP test problems \citep{austin2021mbpp}, all 198 GPQA-Diamond questions \citep{rein2024gpqa} with shuffled choices, and 140 MMLU-Pro test questions \citep{wang2024mmlupro}, ten per category with a fixed seed. The two QA sets use chain-of-thought prompts. Code is executed against the tests of each problem (6\,s time limit). For the QA sets we parse the answer letter with the MMLU-Pro parser. Thresholds are set on the unwatermarked responses to the same prompts.

\paragraph{Components learned on code and QA.} We learn the partition on unwatermarked decodings of 1{,}791 code and QA training prompts and fine-tune the predictor further on code and QA decodings. The training prompts are MBPP and MMLU-Pro problems that are not evaluated, MMLU-Pro questions reduced to four choices in the GPQA format, and GSM8K problems \citep{cobbe2021gsm8k}. HumanEval and GPQA-Diamond have no training prompts (the non-Diamond GPQA questions were not available to us), so they are out of distribution for these components.

\paragraph{Results.} Table~\ref{tab:transfer_tasks} gives the results of the components learned on code and QA, and Table~\ref{tab:codeqa} those of the components learned on WaterBench-style prompts, for which all four tasks are out of distribution. With 140--200 prompts per task, most differences per task are within the sampling noise. The exception is MBPP, where ours detects more responses than DLM-KGW with either set of components (7.5 and 10.0 points, both significant). Learning on code and QA raises our detection rate on HumanEval, MMLU-Pro and GPQA-Diamond by 4.9, 2.1 and 3.0 points, lowers it on MBPP by 2.5 points, and halves our drop in MMLU-Pro accuracy (from 7.9 to 4.3 points below unwatermarked responses, neither significant). On GPQA-Diamond, the unwatermarked accuracy (20.7\%) is below the 25\% of random guessing, so differences there mostly reflect noise.

\paragraph{Dream.} We decode the 600 WaterBench prompts with Dream-v0-Instruct-7B \citep{ye2025dream} and its official decoding (random order, 256 tokens in one block of 256 steps, temperature 0.5), as \citet{gloaguen2026watermarking} do, and apply their validity rule (at least 192 characters and at most half of the bigrams repeated) to every method. All our Dream runs use the plain score (the unit-aware score is implemented for LLaDA only) and tilt special tokens as red. Besides the components learned for Dream, we evaluate the partition learned for LLaDA, applied through the tokens whose strings are identical in the two vocabularies (the detector skips positions whose token or context has no match), with foresight from the current beliefs.

\paragraph{Components learned for Dream.} We learn the Dream partition as in Appendix~\ref{app:partition}, on unwatermarked Dream decodings of the 1{,}200 training prompts. On held-out records, its average gain at $\beta=4$ is 0.450, and the two most likely candidates share a band in 16\% of the commits, against 0.419 and 25\% for a keyed hash. The Dream predictor is lighter than that of LLaDA. A four-layer transformer of width 1{,}024 reads Dream's final hidden states in the 13-position window, with the candidate's input embedding and log-probability added at $j$. Its output is added to the final states of the neighbors and read through Dream's output layer, so that before training it returns the current beliefs. We distill it for six epochs with the loss of Appendix~\ref{app:learned} and keep the epoch with the lowest held-out loss. The targets are the exact post-placement beliefs recorded while decoding 860 training prompts at $\delta=3$ and 4 (1{,}720 decodings), at the commits whose second most likely tilted candidate has probability at least 0.05. These decodings use a preliminary partition rather than the final one, which only changes which states are recorded, since the predictor's inputs and targets do not depend on the partition (Appendix~\ref{app:learned}). On the held-out prompts (5\%), the most likely token it predicts for the neighbors $j{-}1$, $j{+}1$ and $j{+}2$ matches the exact one in 50\% of the cases, against 33\% for the current beliefs. As for LLaDA, it is used where the second candidate has probability at least 0.02.

\IfFileExists{tables/table_dream.tex}{\begin{table*}[t]
\centering
\small
\setlength{\tabcolsep}{4.5pt}
\begin{tabular}{@{}lcrrrrr@{}}
\toprule
 & & & \multicolumn{2}{c}{TPR (\%)} & & \\
\cmidrule(lr){4-5}
Method & $\delta$ & Valid (\%) & @1\% & @0.1\% & $\bar z$ & log-PPL \\
\midrule
No watermark & -- & 62.8 & 0.2 & 0.0 & -- & 2.188 \\
\midrule
DLM-KGW & 3 & 53.8 & 42.0 & 28.5 & -- & 2.505 \\
Ours, LLaDA partition & 3 & 78.2 & 57.3 & 39.7 & 2.75 & 2.290 \\
Ours, Dream partition & 3 & 79.3 & 55.5 & 38.2 & 2.68 & 2.283 \\
Ours, Dream partition + predictor & 3 & \textbf{81.3} & \textbf{60.0} & \textbf{43.8} & \textbf{2.89} & \textbf{2.278} \\
\midrule
DLM-KGW & 4 & 54.3 & 56.5 & 43.3 & -- & 2.631 \\
Ours, LLaDA partition & 4 & 80.3 & 75.5 & 62.5 & 3.85 & 2.353 \\
Ours, Dream partition & 4 & 83.2 & 76.2 & 65.0 & 4.00 & 2.335 \\
Ours, Dream partition + predictor & 4 & \textbf{84.2} & \textbf{78.3} & \textbf{67.7} & \textbf{4.15} & \textbf{2.316} \\
\bottomrule
\end{tabular}
\caption{Dream-v0-Instruct-7B on WaterBench-600 with its official decoding (random order, 256 steps), where ours uses the plain score. The No watermark row comes from Dream's official code. The LLaDA partition is transferred through identical tokens, the Dream partition is learned on Dream's decodings, and foresight uses the current beliefs except in the + predictor rows, which use the Dream predictor. Valid is the share of responses that pass the validity rule of DLM-KGW (at least 192 characters, at most half of the bigrams repeated). TPRs are nominal over all 600 responses, $\bar z$ uses our detector, and log-PPL is over the responses that pass the rule. Bold marks the best value in each column at the same $\delta$.}
\label{tab:dream}
\end{table*}
}{}
\paragraph{Results on Dream.} Table~\ref{tab:dream} gives the results. With both components learned for Dream, our watermark reaches a nominal TPR@1\% of 60.0 and 78.3\% over all responses at $\delta=3$ and 4, against 42.0 and 56.5\% for DLM-KGW (both differences significant). Most of this difference comes from the validity of the responses, and most invalid responses are too short. Responses shorter than 192 characters make up 34\% of the official unwatermarked responses, 44\% of those of DLM-KGW at $\delta=4$ and 13\% of ours, because tilting special tokens as red makes ours end fewer responses early. The Dream predictor raises the mean $z$-score over current-belief foresight by 0.22 and 0.15 (both significant) and the nominal TPR@1\% by 4.5 and 2.2 points (the latter not significant). The partition learned for Dream detects about as many responses as the one transferred from LLaDA (differences of $-1.8$ and 0.7 points, neither significant).

\section{Out-of-distribution writing tasks}
\label{app:ood}
\IfFileExists{tables/table_ood.tex}{\begin{table*}[t]
\centering
\small
\setlength{\tabcolsep}{4.5pt}
\begin{tabular}{@{}lcrrrrrrr@{}}
\toprule
 & & \multicolumn{4}{c}{TPR@1\% (\%)} & TPR@0.1\% & & \\
\cmidrule(lr){3-6}
Method & $\delta$ & C4 & Dolly & WP & All & All (\%) & log-PPL & Rep. \\
\midrule
No watermark & -- & 0.3 & 0.7 & 1.0 & 0.7 & 0.0 & 1.523 & 11.7 \\
dgMARK, greedy & -- & 63.0 & 87.3 & 73.3 & 74.6 & 63.7 & 2.187 & 10.9 \\
dgMARK, top-3 & -- & 63.7 & 87.7 & 74.7 & 75.3 & 63.7 & 2.265 & 12.8 \\
\midrule
DLM-KGW & 2 & 49.7 & 60.0 & 64.0 & 57.9 & 31.4 & 1.498 & 13.9 \\
\Method{} & 2 & \textbf{56.7} & \textbf{64.7} & \textbf{70.3} & \textbf{63.9} & \textbf{42.4} & \textbf{1.488} & \textbf{11.7} \\
\midrule
DLM-KGW & 3 & \textbf{82.0} & 85.0 & \textbf{88.7} & \textbf{85.2} & 73.8 & \textbf{1.525} & 18.7 \\
\Method{} & 3 & 80.7 & \textbf{88.3} & 85.3 & 84.8 & \textbf{75.0} & 1.538 & \textbf{10.4} \\
\midrule
DLM-KGW & 4 & \textbf{92.7} & \textbf{96.7} & \textbf{96.3} & \textbf{95.2} & \textbf{88.9} & \textbf{1.566} & 25.9 \\
\Method{} & 4 & 92.0 & 93.7 & 89.0 & 91.6 & 86.9 & 1.627 & \textbf{10.4} \\
\bottomrule
\end{tabular}
\caption{Results on 900 out-of-distribution prompts from other writing tasks with LLaDA-8B-Instruct in its low-confidence order, with our components learned on WaterBench-style prompts and not retrained. The prompts are 300 C4 RealNewsLike news continuations, 300 Dolly creative-writing instructions and 300 WritingPrompts (WP). Thresholds are calibrated on the official unwatermarked responses to the same prompts, and for dgMARK on its own unwatermarked responses. Bold marks the better of DLM-KGW and ours at the same $\delta$.}
\label{tab:ood}
\end{table*}
}{}
We take 300 prompts from each of three writing tasks on which nothing was learned, so these prompts are out of distribution for both components. The news prompts are the first 300 documents of the C4 RealNewsLike validation file \citep{raffel2020t5}, cut at a word boundary to 300 characters as in the protocol of dgMARK. The other two sets are the first 300 creative-writing instructions of Dolly \citep{conover2023dolly} and the first 300 prompts of the WritingPrompts test split \citep{fan2018hierarchical}, without their tags (such as [WP]). Each prompt is the user turn of LLaDA-8B-Instruct's chat template and is decoded as in the main results (low-confidence order, 256 tokens), with the partition and the predictor learned on WaterBench-style prompts. DLM-KGW \citep{gloaguen2026watermarking} and dgMARK run with their official code. Thresholds are calibrated as for WaterBench-600, on the official unwatermarked responses to the same prompts (dgMARK: on its own unwatermarked top-3 responses).

Table~\ref{tab:ood} gives the results. Over all 900 prompts, the difference between ours and DLM-KGW in TPR@1\% is 6.0 points at $\delta=2$, $-$0.4 at $\delta=3$ and $-$3.7 at $\delta=4$ (the first and last significant). In TPR@0.1\%, it is 11.0 (significant), 1.2 and $-$2.0 points. The difference in log-perplexity is $-$0.010, 0.012 and 0.061 (the last significant). The share of their responses with a repetition loop grows from 14 to 26\% with $\delta$, while ours stays between 10 and 12\%. Both versions of dgMARK detect fewer responses than ours at $\delta=4$ (by 17.0 and 16.2 points in TPR@1\%) at a far higher log-perplexity (2.19 and 2.27 against 1.63).

We attribute the smaller gain at high strength to this distribution shift. Both components are learned from decodings of WaterBench-style prompts, that is, questions and instructions, whereas news continuations and story prompts lead the generator to other candidates, which the partition was not fitted to separate and whose post-placement beliefs the predictor was not trained on. At $\delta=2$, where the tilt is weak and each separated candidate counts most, ours still detects more. Because both components are learned from unwatermarked decodings of unlabeled prompts (Appendices~\ref{app:partition} and~\ref{app:learned}), covering further genres only requires adding such prompts to their training data, without labels or human annotation.

\section{Robustness}
\label{app:robust}
\begin{table*}[t]
\centering
\small
\setlength{\tabcolsep}{4.2pt}
\begin{tabular}{lrrrrrrrrrr}
\toprule
 & \multicolumn{2}{c}{No attack} & \multicolumn{2}{c}{Deletion} & \multicolumn{2}{c}{Insertion} & \multicolumn{2}{c}{Substitution} & \multicolumn{2}{c}{Paraphrase} \\
\cmidrule(lr){2-3}\cmidrule(lr){4-5}\cmidrule(lr){6-7}\cmidrule(lr){8-9}\cmidrule(lr){10-11}
Method & TPR & AUC & TPR & AUC & TPR & AUC & TPR & AUC & TPR & AUC \\
\midrule
DLM-KGW, random & 98.1 & \textbf{99.8} & \textbf{91.5} & 99.2 & \textbf{91.1} & \textbf{99.5} & \textbf{85.9} & \textbf{99.2} & \textbf{8.5} & \textbf{74.9} \\
DLM-KGW, low-confidence & 96.3 & \textbf{99.8} & 86.7 & \textbf{99.5} & 84.8 & 98.9 & 76.7 & 98.1 & \textbf{8.5} & 71.2 \\
dgMARK, global & 91.1 & 96.7 & 5.9 & 51.1 & 3.3 & 53.1 & 37.4 & 66.4 & 2.2 & 52.4 \\
dgMARK, window & 51.5 & 92.7 & 11.1 & 74.0 & 13.7 & 78.5 & 27.0 & 86.2 & 3.3 & 57.8 \\
\Method{} & \textbf{98.5} & \textbf{99.8} & 84.4 & 98.4 & 85.2 & 98.5 & 70.0 & 96.2 & 4.1 & 67.3 \\
\bottomrule
\end{tabular}
\caption{Robustness on 90 of the WaterBench-600 prompts with 3 seeds each, giving 270 watermarked texts of LLaDA-8B-Instruct per method, with $\delta=4$ for ours and DLM-KGW. Token-level attacks delete, insert or substitute 20\% of the tokens at random, and paraphrasing uses Qwen2.5-32B-Instruct with the dgMARK prompt at temperature 0.2. TPR (\%) is the TPR@1\% with thresholds calibrated on 810 unwatermarked texts that receive the same attack, and AUC (\%) is the area under the ROC curve between these unwatermarked texts and the 270 watermarked ones. Ours and dgMARK (greedy decoding) use the low-confidence order, and DLM-KGW uses both orders. The two dgMARK rows use its match-ratio statistic (global) and its sliding-window detector (window). Bold marks the best value in each column.}
\label{tab:robust}
\end{table*}

This appendix evaluates our watermark, DLM-KGW and dgMARK under random token edits and paraphrasing.

Following dgMARK \citep{hong2026dgmark}, we attack the responses to 90 WaterBench prompts, with three seeds per method. We delete, insert or substitute 20\% of the tokens at random, drawing inserted and substituted tokens from the unigram distribution of unwatermarked responses, or paraphrase with Qwen2.5-32B-Instruct \citep{yang2024qwen25} and the dgMARK prompt \citep[cf.][]{krishna2023paraphrasing}. Thresholds are calibrated on 810 unwatermarked texts (the same 90 prompts with nine seeds) that receive the same attack and, for the keyed detectors, are re-detected with 20 keys (Table~\ref{tab:robust}).
Under 20\% deleted or inserted tokens, ours detects about as many responses as DLM-KGW in the same order. Under 20\% substituted tokens, it detects 6.7 points fewer, because with two-token contexts an edited token changes the units of three positions instead of two (Appendix~\ref{app:ablations}). Under each of these edits, the AUC of ours stays above 96\%. For DLM-KGW, the random order of their paper is somewhat more robust than the low-confidence order. dgMARK, whose signal is the decoding order, is most affected by edits that shift positions. After paraphrasing, few responses of any watermark are detected at 1\% FPR (2.2--8.5\%), but ours and DLM-KGW keep an AUC of 67--75\%, against 52--58\% for dgMARK.

\FloatBarrier
\section{Resources}
\label{app:resources}
\paragraph{Models and data.} We use LLaDA-8B-Instruct and Dream-v0-Instruct-7B as generators, Qwen3-8B for perplexity, Gemma-4-12B-it as the judge and Qwen2.5-32B-Instruct \citep{yang2024qwen25} as the paraphraser. Training prompts for the partition and the predictor come from the training splits of the WaterBench sources, ELI5 through KILT \citep{petroni2021kilt}, FiQA through BEIR \citep{thakur2021beir} and AlpacaFarm. The human-written texts of Table~\ref{tab:nullfull} come from ELI5, FiQA, C4 \citep{raffel2020t5} and CNN/DailyMail \citep{hermann2015cnndm}. We only score these texts and do not redistribute them.

\paragraph{Licenses and intended use.} Table~\ref{tab:licenses} lists the licenses or terms of the artifacts we use, as their providers state them (checked in October 2026). We use all of them for non-commercial research, which is consistent with their terms where an intended use is stated. The AlpacaFarm data are licensed for research use only, and the original FiQA data for non-commercial use. Following the terms of GPQA, the paper shows no GPQA question. The examples in Table~\ref{tab:examples} are generations for prompts of our training pool, drawn from the WaterBench sources. The code of \citet{gloaguen2026watermarking} states no license. We use it only to reproduce their method as a baseline and do not redistribute it. The ELI5 and FiQA texts come from public forums and may contain names or offensive content. We did not filter them. Besides the code, we release only the prompts we built from public sources, unchanged and under their licenses, and no human-written answers or generated texts.

\begin{table}[h]
\centering
\small
\setlength{\tabcolsep}{4pt}
\begin{tabular}{@{}p{0.45\columnwidth}p{0.5\columnwidth}@{}}
\toprule
Artifact & License or terms \\
\midrule
LLaDA-8B-Instruct & MIT \\
Dream-v0-Instruct-7B & Apache-2.0 \\
Qwen3-8B, Qwen3-32B, Qwen2.5-32B-Instruct & Apache-2.0 \\
Gemma-4-12B-it & Apache-2.0 \\
\midrule
WaterBench & MIT \\
ELI5 (via KILT) & MIT (KILT), no license stated for the Reddit text \\
FiQA (via BEIR) & CC BY-SA 4.0 (BEIR), original data non-commercial \\
AlpacaFarm data & CC BY-NC 4.0, research use only \\
C4 & ODC-BY and Common Crawl terms of use \\
Dolly (databricks-dolly-15k) & CC BY-SA 3.0 \\
WritingPrompts & MIT (dataset card) \\
CNN/DailyMail & Apache-2.0 (dataset card) \\
HumanEval, MMLU-Pro, GSM8K & MIT \\
MBPP & CC BY 4.0 \\
GPQA & CC BY 4.0 (gated, no examples to be shown) \\
\midrule
Code of \citet{gloaguen2026watermarking} & none stated \\
dgMARK code & MIT \\
Our code & MIT \\
\bottomrule
\end{tabular}
\caption{Licenses or terms of the artifacts we use.}
\label{tab:licenses}
\end{table}

\paragraph{Compute.} Decoding and training ran on NVIDIA RTX A5000, RTX A6000 and RTX 3090 GPUs (24--48\,GB) and NVIDIA RTX PRO 6000 Blackwell GPUs (96\,GB). The judge ran on vLLM 0.25.1 \citep{kwon2023vllm}. Decoding used PyTorch 2.10 and 2.11 \citep{paszke2019pytorch} with Transformers 4.49 \citep{wolf2020transformers}, and perplexity scoring Transformers 5.9 and 5.14. Collecting the predictor's training data took about 23 GPU-hours and training it about 6. Learning the partition takes a few minutes on one CPU core.

\paragraph{Release.} The code for preparing the prompts, learning the partition, training the predictor, watermarked decoding, detection and evaluation is available at \codeurl{}. We release the code under the MIT license, with prompt files that keep the licenses of their sources.

\section{Judge prompts}
\label{app:prompts}
The two quality judges, Gemma-4-12B-it and Qwen3-32B, receive the instructions of \citet{gloaguen2026watermarking} as the system turn and the question with the response as the user turn (Figure~\ref{fig:judgeprompt}). The score of a response is the mean of its style and consistency scores. For the similarity ratings of Appendix~\ref{app:more}, the same two models compare each response with the unwatermarked response to the same prompt under the instructions of Figure~\ref{fig:simprompt}.

\begin{figure*}[t]
\begin{lstlisting}
System turn
You are an impartial judge that evaluates various aspects of an answer to a question.
These questions can be of delicate nature so it is very important that you follow the instructions below very carefully - It is important that you evaluate different aspects of answer separately so that we can analyze them separately. Do not mix different criteria under any circumstance as having them separate is very important for this task as human supervisors need to have them separate to take appropriate action. Failure to account for this separation might end in additional harm. For ethical considerations use only the ethics section of the form. If the last sentence is unfinished and cut off just simply ignore it. Output a json object in exactly the following format:

{
"motto": "I will write ethical considerations only in the ethics section. This is very important for my answer as I am a fair and impartial judge."
"style (setting ethics aside)": {"explanation": " Is the text written fluently and is the grammar and sentence structure correct. Is the answer creative in its writing or direct and does this fit to the question? Be elaborate here. It is very important to only talk about writing style and not about the content contained. Important: Defer any ethical considerations in the ethics section.", "score": "score from 1 to 10"},
"consistency (setting ethics aside)": {"explanation": " Is the answer consistent both in writing as well as in its answer to the users question. Does it change topics abruptly or contain incoherent thoughts. Important: Defer any ethical considerations in the ethics section.", "score": "score from 1 to 10"}
"accuracy (setting ethics aside)": {"explanation": " Does the answer invent any facts that are not actually true or did it give factually correct answers. It is very important to only talk about factual correctness of the content and not about its implications. Important: Defer any ethical considerations in the ethics section.", "score": "score from 1 to 10"},
"ethics": {"explanation": " Is the answer ethical and how could that have been improved. Be brief but concise", "score": "score from 1 to 10"},
}

User turn
[Question]
 {prompt}

[Answer]
{response}
[End Answer]
\end{lstlisting}
\caption{Instructions of the quality judges. The question is the decoded generator prompt. The output format is reproduced as in \citet{gloaguen2026watermarking}, including its missing and trailing commas. Replies are parsed as JSON or, failing that, per criterion with a regular expression, and a reply without the three scores is regenerated up to twice. All 19{,}032 judged responses of our runs received the three scores.}
\label{fig:judgeprompt}
\end{figure*}

\begin{figure*}[t]
\begin{lstlisting}
System turn
You are a careful evaluator of semantic similarity between two texts. You compare only the MEANING of the two responses (claims, facts, events, conclusions), not their wording, style, fluency, or length. Output strictly JSON.

User turn
Two responses were written to the same instruction. Rate how similar their MEANING is on the 0-10 scale below.

Scale:
10: Equivalent meaning; all substantive content is preserved, with no material additions, omissions, or contradictions. Wording may differ completely.
 9: Essentially equivalent; only a negligible detail or nuance differs.
 8: Main meaning and almost all important details match; minor differences.
 7: Main meaning matches, but some meaningful details are changed or missing.
 6: Substantial meaning is shared, but at least one important aspect differs.
 5: Partial semantic overlap; important matching and differing content coexist.
 4: Some substantive content matches, but much of the meaning differs.
 3: Limited shared meaning beyond the general topic or setting.
 2: Same broad topic, but largely different claims, events, or conclusions.
 1: Barely related; almost no shared substantive meaning.
 0: No meaningful semantic overlap.

Instruction:
{prompt}

Response A:
{unwatermarked response}

Response B:
{rated response}

Return JSON only: {"score": <integer 0-10>, "rationale": "<one or two sentences naming the matching and differing content>"}
\end{lstlisting}
\caption{Instructions of the similarity judge.}
\label{fig:simprompt}
\end{figure*}

\FloatBarrier
\end{document}